\documentclass[twoside,11pt]{article}

\usepackage{jmlr2e}
\usepackage{lastpage}

\usepackage{ifluatex}

\ifluatex
\usepackage{fontspec}
\else
\usepackage[utf8]{inputenc}
\usepackage[T1]{fontenc}
\fi
\usepackage{newunicodechar}

\usepackage[dvipsnames]{xcolor}
\usepackage{microtype}
\usepackage[super]{nth}
\usepackage{url}

\usepackage{enumitem}

\usepackage{caption}
\usepackage[newfloat,frozencache]{minted}

\usepackage{booktabs}
\usepackage{multirow}
\usepackage{tablefootnote}
\usepackage{tabularx}
\usepackage[export]{adjustbox}

\usepackage{algorithm}
\usepackage{algpseudocode}

\algrenewcommand\algorithmicindent{0.7em}
\algnewcommand\algorithmicinput{\textbf{Input:}}
\algnewcommand\algorithmicoutput{\textbf{Output:}}
\algnewcommand\Input{\item[\algorithmicinput]}
\algnewcommand\Output{\item[\algorithmicoutput]}
\algrenewcommand{\algorithmiccomment}[1]{\hfill\small{\textcolor{darkgray}{\(\triangleright\) #1}}}

\usepackage{amsfonts}
\usepackage{amsmath}
\usepackage{amssymb}

\usepackage{mathtools}
\usepackage{thmtools}
\usepackage{thm-restate}

\usepackage{xparse}

\definecolor{easy_blue}{RGB}{0,98,163}%
\definecolor{easy_red}{RGB}{150,60,0}
\definecolor{easy_blue_bg}{RGB}{234,241,248}
\definecolor{easy_red_bg}{RGB}{253,246,235}

\definecolor{blue_bg}{RGB}{0,61,86}

\usepackage{soul}
\DeclareRobustCommand{\hlblue}[1]{{\sethlcolor{easy_blue_bg}\hl{#1}}}
\DeclareRobustCommand{\hlred}[1]{{\sethlcolor{easy_red_bg}\hl{#1}}}

\usepackage{hyperref}%
\hypersetup{
hidelinks%
}

\usepackage[page]{appendix}

\usepackage[capitalise,noabbrev]{cleveref}%
\crefname{equation}{}{}
\Crefname{assumption}{Assumption}{Assumptions}

\allowdisplaybreaks{}

\NewDocumentCommand{\G}{e{_^}}{%
\mathop{}\!%
\nabla{}
\IfValueT{#1}{_{\!#1}}%
\IfValueT{#2}{^{#2}}%
}
\newcommand{\blank}{{\mspace{2mu}\cdot\mspace{2mu}}}
\newcommand{\argmax}{\operatorname*{arg\,max}}
\newcommand{\argmin}{\operatorname*{arg\,min}}

\newcommand*{\tr}{^{\mkern-1.5mu\mathsf{T}}}%
\newcommand*{\hc}{^{\mathsf{H}}}%

\newcommand{\E}{\operatorname*{\mathbb{E}}}
\newcommand{\V}{\operatorname*{\mathbb{V}}}
\newcommand{\C}{\operatorname*{\mathbb{C}}}
\newcommand{\ind}{\operatorname*{\perp\!\!\!\perp}}
\newcommand{\iid}{\text{\tiny{iid}}}

\newunicodechar{…}{\dots}
\newunicodechar{•}{\bullet}%
\newunicodechar{□}{\blank}%

\newunicodechar{×}{\times}
\newunicodechar{∂}{\partial}
\newunicodechar{∇}{\G}
\newunicodechar{∫}{\int}
\newunicodechar{∑}{\sum}
\newunicodechar{∏}{\prod}

\newunicodechar{¬}{\neg}
\newunicodechar{∅}{\varnothing}
\newunicodechar{∈}{\in}
\newunicodechar{∉}{\nin}
\newunicodechar{∃}{\exists}
\newunicodechar{∄}{\nexists}
\newunicodechar{∀}{\forall}
\newunicodechar{⊂}{\subset}
\newunicodechar{⊃}{\supset}
\newunicodechar{⊆}{\subseteq}
\newunicodechar{⊇}{\supseteq}
\newunicodechar{∪}{\cup}
\newunicodechar{⋓}{\bigcup}
\newunicodechar{∩}{\cap}
\newunicodechar{⋒}{\bigcap}

\newunicodechar{⊕}{\oplus}
\newunicodechar{⊖}{\ominus}
\newunicodechar{⊗}{\otimes}
\newunicodechar{⊘}{\oslash}
\newunicodechar{⊙}{\odot}

\newunicodechar{←}{\leftarrow}
\newunicodechar{→}{\rightarrow}
\newunicodechar{⟵}{\longleftarrow}
\newunicodechar{⟶}{\longrightarrow}
\newunicodechar{⇒}{\Rightarrow}
\newunicodechar{⇐}{\Leftarrow}
\newunicodechar{↦}{\mapsto}

\newunicodechar{≔}{\coloneqq}
\newunicodechar{≈}{\approx}
\newunicodechar{∝}{\propto}
\newunicodechar{≠}{\neq}

\newunicodechar{ℜ}{\Re}
\newunicodechar{ℑ}{\Im}
\newunicodechar{∞}{\infty}

\newunicodechar{≪}{\ll}
\newunicodechar{≫}{\gg}
\newunicodechar{≤}{\leq}
\newunicodechar{≥}{\geq}

\newunicodechar{∼}{\sim}
\newunicodechar{⫫}{\ind}
\newunicodechar{‖}{\Vert}

\newunicodechar{⋅}{\cdot}%
\newunicodechar{ᵀ}{\tr}
\newunicodechar{ᴴ}{\hc}

\newunicodechar{Δ}{\Delta}
\newunicodechar{Γ}{\Gamma}
\newunicodechar{Λ}{\Lambda}
\newunicodechar{Ω}{\Omega}
\newunicodechar{Φ}{\Phi}
\newunicodechar{Π}{\Pi}
\newunicodechar{Ψ}{\Psi}
\newunicodechar{Σ}{\Sigma}
\newunicodechar{Θ}{\Theta}
\newunicodechar{ϒ}{\Upsilon}
\newunicodechar{Ξ}{\Xi}
\newunicodechar{α}{\alpha}
\newunicodechar{β}{\beta}
\newunicodechar{χ}{\chi}
\newunicodechar{δ}{\delta}
\newunicodechar{ϵ}{\epsilon}
\newunicodechar{η}{\eta}
\newunicodechar{γ}{\gamma}
\newunicodechar{ι}{\iota}
\newunicodechar{κ}{\kappa}
\newunicodechar{λ}{\lambda}
\newunicodechar{μ}{\mu}
\newunicodechar{ν}{\nu}
\newunicodechar{ω}{\omega}
\newunicodechar{ϕ}{\phi}
\newunicodechar{π}{\pi}
\newunicodechar{ψ}{\psi}
\newunicodechar{ρ}{\rho}
\newunicodechar{σ}{\sigma}
\newunicodechar{τ}{\tau}
\newunicodechar{θ}{\theta}
\newunicodechar{υ}{\upsilon}
\newunicodechar{ε}{\varepsilon}
\newunicodechar{φ}{\varphi}
\newunicodechar{ϖ}{\varpi}
\newunicodechar{ϱ}{\varrho}
\newunicodechar{ς}{\varsigma}
\newunicodechar{ϑ}{\vartheta}
\newunicodechar{ξ}{\xi}
\newunicodechar{ζ}{\zeta}

\newunicodechar{𝛥}{\varDelta}
\newunicodechar{𝛤}{\varGamma}
\newunicodechar{𝛬}{\varLambda}
\newunicodechar{𝛺}{\varOmega}
\newunicodechar{𝛷}{\varPhi}
\newunicodechar{𝛱}{\varPi}
\newunicodechar{𝛹}{\varPsi}
\newunicodechar{𝛴}{\varSigma}
\newunicodechar{𝛩}{\varTheta}
\newunicodechar{𝛶}{\varUpsilon}
\newunicodechar{𝛯}{\varXi}

\newunicodechar{𝔸}{\mathbb{A}}
\newunicodechar{𝔹}{\mathbb{B}}
\newunicodechar{ℂ}{\C}
\newunicodechar{𝔻}{\mathbb{D}}
\newunicodechar{𝔼}{\E}
\newunicodechar{𝔽}{\mathbb{F}}
\newunicodechar{𝔾}{\mathbb{G}}
\newunicodechar{ℍ}{\mathbb{H}}
\newunicodechar{𝕀}{\mathbb{I}}
\newunicodechar{𝕁}{\mathbb{J}}
\newunicodechar{𝕂}{\mathbb{K}}
\newunicodechar{𝕃}{\mathbb{L}}
\newunicodechar{𝕄}{\mathbb{M}}
\newunicodechar{ℕ}{\mathbb{N}}
\newunicodechar{𝕆}{\mathbb{O}}
\newunicodechar{ℙ}{\mathbb{P}}
\newunicodechar{ℚ}{\mathbb{Q}}
\newunicodechar{ℝ}{\mathbb{R}}
\newunicodechar{𝕊}{\mathbb{S}}
\newunicodechar{𝕋}{\mathbb{T}}
\newunicodechar{𝕌}{\mathbb{U}}
\newunicodechar{𝕍}{\V}
\newunicodechar{𝕎}{\mathbb{W}}
\newunicodechar{𝕏}{\mathbb{X}}
\newunicodechar{𝕐}{\mathbb{Y}}
\newunicodechar{ℤ}{\mathbb{Z}}

\newunicodechar{𝒜}{\mathcal{A}}
\newunicodechar{ℬ}{\mathcal{B}}
\newunicodechar{𝒞}{\mathcal{C}}
\newunicodechar{𝒟}{\mathcal{D}}
\newunicodechar{ℰ}{\mathcal{E}}
\newunicodechar{ℱ}{\mathcal{F}}
\newunicodechar{𝒢}{\mathcal{G}}
\newunicodechar{ℋ}{\mathcal{H}}
\newunicodechar{ℐ}{\mathcal{I}}
\newunicodechar{𝒥}{\mathcal{J}}
\newunicodechar{𝒦}{\mathcal{K}}
\newunicodechar{ℒ}{\mathcal{L}}
\newunicodechar{ℳ}{\mathcal{M}}
\newunicodechar{𝒩}{\mathcal{N}}
\newunicodechar{𝒪}{\mathcal{O}}
\newunicodechar{𝒫}{\mathcal{P}}
\newunicodechar{𝒬}{\mathcal{Q}}
\newunicodechar{ℛ}{\mathcal{R}}
\newunicodechar{𝒮}{\mathcal{S}}
\newunicodechar{𝒯}{\mathcal{T}}
\newunicodechar{𝒰}{\mathcal{U}}
\newunicodechar{𝒱}{\mathcal{V}}
\newunicodechar{𝒲}{\mathcal{W}}
\newunicodechar{𝒳}{\mathcal{X}}
\newunicodechar{𝒴}{\mathcal{Y}}
\newunicodechar{𝒵}{\mathcal{Z}}

\jmlrheading{24}{2023}{1-\pageref{LastPage}}{4/23; Revised 9/23}{10/23}{23-0527}{Stefano Peluchetti}

\ShortHeadings{Diffusion Bridge Mixture Transports}{Peluchetti}
\firstpageno{1}

\newcommand{\RP}{\overline{P}}
\newcommand{\RF}{\overline{F}}
\newcommand{\RX}{\overline{X}}
\newcommand{\RM}{\overline{M}}
\newcommand{\RA}{\overline{A}}
\newcommand{\RPi}{\overline{Π}}
\newcommand{\RC}{\overline{C}}
\newunicodechar{𝔯}{\mathfrak{r}}
\newcommand{\TO}{\mathbin{‖}}
\newcommand{\DynSB}{(\mathrm{S_{dyn}})}
\newcommand{\StaSB}{(\mathrm{S_{sta}})}
\newcommand{\DynHSB}{(\mathrm{H_{dyn}})}
\newcommand{\StaHSB}{(\mathrm{H_{sta}})}

\begin{document}

\title{Diffusion Bridge Mixture Transports, Schrödinger Bridge Problems and Generative Modeling}

\author{\name{Stefano Peluchetti} \email{speluchetti@cogent.co.jp}\\
\addr{Cogent Labs\\
106--0032, Tokyo, Japan.}}

\editor{Aapo Hyvärinen}

\maketitle

\begin{abstract}%
The dynamic Schrödinger bridge problem seeks a stochastic process that defines a transport between two target probability measures, while optimally satisfying the criteria of being closest, in terms of Kullback-Leibler divergence, to a reference process.
We propose a novel sampling-based iterative algorithm, the iterated diffusion bridge mixture (IDBM) procedure, aimed at solving the dynamic Schrödinger bridge problem.
The IDBM procedure exhibits the attractive property of realizing a valid transport between the target probability measures at each iteration.
We perform an initial theoretical investigation of the IDBM procedure, establishing its convergence properties.
The theoretical findings are complemented by numerical experiments illustrating the competitive performance of the IDBM procedure.
Recent advancements in generative modeling employ the time-reversal of a diffusion process to define a generative process that approximately transports a simple distribution to the data distribution.
As an alternative, we propose utilizing the first iteration of the IDBM procedure as an approximation-free method for realizing this transport.
This approach offers greater flexibility in selecting the generative process dynamics and exhibits accelerated training and superior sample quality over larger discretization intervals.
In terms of implementation, the necessary modifications are minimally intrusive, being limited to the training loss definition.
\end{abstract}


\section{Introduction}

Generating samples from a given distribution is a central topic in computational statistics and machine learning.
In recent years, transportation of probability measures has gained considerable attention as a computational tool for sampling.
The underlying concept entails generating samples from a complex distribution by first obtaining samples from a simpler distribution and subsequently computing a suitably constructed map of these samples \citep{elmoselhy2012bayesian,marzouk2016sampling}.
Let us consider the following definition: given two probability measures $Γ$ and $ϒ$, a transport from $Γ$ to $ϒ$ is defined as a map $h(X,ε)$, where $ε$ is an independent auxiliary random element, such that if $X ∼ Γ$ (i.e., $X$ is distributed according to $Γ$), then $Y = h(X,ε)$ satisfies $Y ∼ ϒ$.
This broad definition encompasses both deterministic and random transports and includes a wide range of sampling methods, such as Markov Chain Monte Carlo \citep{brooks2011handbook}, normalizing flows \citep{papamakarios2021normalizing}, and score-based generative modeling \citep{ho2020denoising,song2021scorebased}.

To narrow down the extensive class of possible transports, we can identify the optimal transport according to a suitably chosen criterion.
This paper addresses a more generalized version of the optimization problem initially posed by Schrödinger in 1931 \citep{leonard2014survey}.
In this context, both $Γ$ and $ϒ$ are defined on the same $d$-dimensional Euclidean space.
The optimization space comprises $d$-dimensional diffusion processes on the time interval $[0,τ]$, constrained to have initial distribution $Γ$ and terminal distribution $ϒ$.
Optimality is achieved by minimizing the Kullback-Leibler divergence to a reference diffusion process.
The optimal map $Y = h(X,ε)$ is thus obtained by defining $ε$ as the $d$-dimensional Brownian motion on $[0,τ]$ which drives the optimal diffusion process started at $X$, with terminal value $Y$.
This seemingly intricate procedure of defining a transport is of practical relevance.
When the reference diffusion is given by a scaled Brownian motion $σ W_t$, for some $σ > 0$, $(X,h(X,ε))$ solves the Euclidean entropy-regularized optimal transport (EOT) problem \citep[Chapter 4]{peyre2020computational} for the regularization level $2σ^2$.
EOT provides a more tractable alternative to solving the optimal transport (OT) problem.
However, in high-dimensional settings, solving the EOT problem remains challenging.
The dynamic formulation, achieved by inflating the original space where $Γ$ and $ϒ$ are supported so that they correspond to the initial-terminal distributions of a stochastic process, is essential to recent computational advancements in these demanding settings.
Specifically, \citet{bortoli2021diffusion,vargas2021solving} rely on a time-reversal result for diffusion processes to implement sampling-based iterative algorithms aimed at solving the dynamic Schrödinger bridge problem.

\subsection{Our Contributions}

In this work, we begin by reviewing the theory of Schrödinger bridge problems in their various forms, as well as the approaches developed by \citet{bortoli2021diffusion,vargas2021solving}.
We establish that sampling-based time-reversal approaches suffer from a simulation-inference mismatch.
Specifically, at each iteration, samples from the previous iteration are utilized to infer a new diffusion process, but relevant regions of the state space can be left unexplored.
This issue is particularly prominent in scenarios characterized by a low level of randomness in the reference process.
This is problematic when the goal is to solve the OT problem: for the EOT solution to closely approximate the OT solution, the scale $σ$ of the reference diffusion $σ W_t$ must be small.

The primary contribution of this work is the development of a novel sampling-based iterative algorithm, the IDBM procedure, for addressing the dynamic Schrödinger bridge problem.
We begin by noting that the problem's solution, i.e.\ the optimal diffusion process, can be represented as a mixture of diffusion bridges, with the mixing occurring over the bridges' endpoints.
In general, processes constructed in this manner are not diffusion processes \citep{jamison1974reciprocal,jamison1975markov}.
The IDBM procedure involves the following iterative steps:
(i) constructing a stochastic process as a mixture of diffusion bridges such that its initial-terminal distribution forms a coupling of the target probability measures $Γ$ and $ϒ$;
(ii) matching the marginal distributions of the stochastic process generated in (i) with a diffusion process;
(iii) updating the coupling from step (i) with the initial-terminal distributions of the diffusion obtained in step (ii).
The diffusion process of step (ii) is inferred based on samples from the stochastic process of step (i).
Crucially, as both processes share the same marginal distributions, no simulation-inference mismatch occurs.
In this study, we conduct an initial theoretical examination of the IDBM procedure, establishing its convergence properties toward the solution to the dynamic Schrödinger bridge problem.
Additionally, we carry out empirical evaluations of the IDBM procedure in comparison to the IPF procedure, highlighting its robustness.

The recent advancements in score-based generative approaches \citep{song2021scorebased,rombach2022highresolution} have demonstrated remarkable generative quality in visual applications by leveraging the time-reversal of a reference diffusion process to define the sampling process.
To elaborate further, the reference diffusion's dynamics are selected based on a decoupling criterion, which requires the process's terminal value to be approximately independent of the process's initial value.
The reference process's initial distribution is assumed to be the empirical distribution of a dataset of interest, or a slightly perturbed version of it.
The time reversal of the reference diffusion, complemented by a dataset-independent initial distribution, defines the sampling process.
In this work, we propose to utilize the first iteration of the IDBM procedure as an alternative method of defining a sampling process targeting a dataset of interest.
Under this proposal, the reference diffusion's dynamics are no longer constrained by decoupling considerations.
We conduct an empirical evaluation of the resulting transport, demonstrating its competitiveness for visual applications in comparison to the approach of \citet{song2021scorebased}.
Remarkably, the proposed approach exhibits accelerated training and superior sample quality at larger discretization intervals without fine-tuning the involved hyperparameters.
From a practitioner standpoint, implementing the proposed method requires minimal changes, which are confined to the training loss definition.

\subsection{Structure of the Paper}

The paper is structured as follows.
In \cref{sec:schrodiger_bridges} we introduce the Schrödinger bridge problem in its various forms, the IPF procedure, its time-reversal sampling-based implementations for the case of a reference diffusion process, and score-based generative modeling.
The IDBM procedure, its convergence properties, and suitable training objectives are presented in \cref{sec:dbm}.
In \cref{sec:sde_class} we define the class of reference stochastic differential equations (SDE) considered in the applications, which are carried out in \cref{sec:applications}.
\cref{sec:related_work} discusses relevant connections and \cref{sec:discussion} provides concluding remarks for the paper.
Proofs, auxiliary formulae and visualizations, code listings and a table summarizing the notation used throughout this work are deferred to the Appendices.

\subsection{Notation and Preliminaries}

In light of the methodological focus of the present work, we refrain from discussing some of the more technical aspects.
The excellent treatises by \citet{leonard2014survey,leonard2014properties} provide a formal account of the preliminaries discussed below and of the developments of \cref{sec:schrodiger_problems}.
To aid the reader, a summary of the main notation is provided in \cref{tab:symbols} (\cref{app:notation}).

All the stochastic processes we are concerned with are defined over the continuous time interval $[0,τ]$, have state space $ℝ^d$, and are continuous, for some constants $τ > 0$ and $d ≥ 1$.
A realization of a stochastic process, or path, is an element $x ∈ 𝒞([0,τ],ℝ^d)$, the space of continuous functions from $[0,τ]$ to $ℝ^d$.

We denote each probability measure (PM, or distribution, law) with an uppercase letter.
For a PM $P$, $X ∼ P$ denotes that $X$ is a random element with distribution $P$.
Let $n ≥ 1$ and $t_1,…,t_n ∈ [0,τ]$.
We denote the collection of PMs over $𝒞([0,τ],ℝ^d)$ with $𝒫_𝒞$ and the collection of PMs over $ℝ^{d × n}$ with $𝒫_n$.
When $P ∈ 𝒫_n$ admits a density with respect to the Lebesgue measure on $ℝ^{d × n}$ we denote this density with the same lowercase letter (here $p$).
For $P ∈ 𝒫_𝒞$, we denote with $P_{t_1,…,t_n} ∈ 𝒫_n$ the marginalization of $P$ at $t_1,…,t_n$ and with $P_{•|t_1,…,t_n} ∈ 𝒫_𝒞$ the conditioning of $P$ given the values at times $t_1,…,t_n$.
More explicitly, if $X ∼ P ∈ 𝒫_𝒞$, we have $P_{t_1,…,t_n}(A_1,…,A_n) = \Pr[X_{t_1} ∈ A_1, …, X_{t_n} ∈ A_n]$ and $P_{•|t_1,…,t_n}(A|x_1,…,x_n) = \Pr[X ∈ A|X_{t_1} = x_1, …, X_{t_n} = x_n]$.
A path PM can be defined via a mixture: if $Ψ ∈ 𝒫_n$ and $P ∈ 𝒫_𝒞$, let $Q = Ψ P_{•|t_1,…,t_n}$ stand for $Q(A) = ∫P_{•|t_1,…,t_n}(A|x_1,…,x_n)Ψ(dx_1,…,dx_n)$.
That is, $Q$ is obtained by pinning down $X ∼ P$ at times $t_1,…,t_n$, and taking a mixture with respect to $Ψ$ over the pinned down values.
With this notation, the marginal-conditional decomposition of $P ∈ 𝒫_𝒞$  is $P = P_{t_1,…,t_n}P_{•|t_1,…,t_n}$.
If $P,Q ∈ 𝒫_𝒞$, and $P$ is absolutely continuous with respect to $Q$ ($P ≪ Q$) with density $(dP/dQ)(x)$, the marginal-conditional decomposition of $(dP/dQ)(x)$ is $(dP/dQ)(x) = (dP_{t_1,…,t_n}/dQ_{t_1,…,t_n})(x_{t_1},…,x_{t_n}) × (dP_{•|t_1,…,t_n}/dQ_{•|t_1,…,t_n})(x)$.

Given two PMs $Γ,ϒ ∈ 𝒫_1$, we define $𝒫_𝒞(Γ,ϒ) ⊆ 𝒫_𝒞$ as the collection of PMs having $Γ$ as initial distribution and $ϒ$ as terminal distribution.
Hence, $P ∈ 𝒫_𝒞(Γ,ϒ)$ transports $Γ$ to $ϒ$ ($P_0 = Γ, P_τ = ϒ$), while its time reversal transports $ϒ$ to $Γ$.
We write $𝒫_𝒞(Γ,□)$ when only the initial distribution is fixed, and $𝒫_𝒞(□,ϒ)$ when only the terminal distribution is fixed.
We define $𝒫_2(Γ,ϒ) ⊆ 𝒫_2$ as the collection of PMs with prescribed marginal distributions $Γ$ and $ϒ$, i.e.\ the collection of couplings between $Γ$ and $ϒ$.
We write $𝒫_2(Γ,□)$ ($𝒫_2(□,ϒ)$) when only the first (second) marginal distribution is fixed.
For $P ∈ 𝒫_2$, we write $P_0,P_τ$ to denote respectively the first and second marginal, and denote the conditional distributions with $P_{•|0}$ and $P_{•|τ}$.
With this notation, for $P ∈ 𝒫_𝒞 ∪ 𝒫_2$ we have $P = P_0 P_{•|0} = P_τ P_{•|τ}$.

We write $𝒰(0,τ)$ for the uniform distribution on $(0,τ)$, $𝒩_d(μ,Σ)$ for the $d$-variate normal distribution with mean $μ ∈ ℝ^d$ and covariance $Σ ∈ ℝ^{d×d}$.
We write $X_i \overset{\iid}{∼} P$ to denote that each random elements $X_i$ is independent and identically distributed according to a PM $P$.
For $S,R$ two PMs on the same space, the Kullback-Leibler (KL) divergence from $S$ to $R$ is defined as $D_{KL}(S \TO R) ≔ 𝔼_S[\log(\frac{dS}{dR})]$ whenever $S ≪ R$, $+∞$ otherwise.

We denote identity matrices with $I$, and transposition of a square matrix $A$ with $Aᵀ$.
We use the notation $[A]_i$ and $[A]_{i,j}$ to denote indexing respectively of a vector and of a matrix.
The Euclidean norm of a vector $V$ is denoted by $‖V‖$.

If $X ∼ P ∈ 𝒫_𝒞$, we denote the time reversal of $X$ with $\RX$, that is $\RX_t ≔ X_{τ - t}$ ($t ∈ [0,τ]$).
The time reversal of $P$ itself is defined as $\RP[A] ≔ P[\RA]$, $\RA ≔ \{x: x_{τ-t} ∈ A\}$, so that $\RX ∼ \RP$.
Note that $E_P[f(X)] = E_{\RP}[f(\RX)]$ for any integrable $f: 𝒞([0,τ],ℝ^d) → ℝ$.
Thorough the paper we set $𝔯 ≔ τ - t$.

Without further mention, all SDE drift coefficients are assumed to be $ℝ^d$-valued and all SDE diffusion coefficients are assumed to be $ℝ^{d×d}$-valued.
Most SDEs, and associated diffusion solutions, are denoted with the same letter $X$.
To avoid ambiguity, we clearly denote the corresponding probability laws.
Similarly, we denote most $ℝ^d$-valued standard Brownian motions with $W$.
It is understood that all Brownian motions driving different SDEs are independent nonetheless.

\section{Schrödinger Bridge Problems}\label{sec:schrodiger_bridges}

In \cref{sec:schrodiger_problems} we review the theory of Schrödinger bridge problems in both dynamic and static formulations.
We initially consider the generic setting of continuous stochastic processes before specializing to the case of a diffusion reference measures.
In \cref{sec:ipf,sec:ipf_diffusion_approach}, we discuss the classical iterative algorithm employed in the numerical solution to Schrödinger bridge problems, focusing on recent contributions that rely on a time reversal argument for diffusion processes.
Throughout our discussion, we adopt a path space, or continuous-time, perspective.
This perspective not only provides clarity but also highlights the connection with a sequential estimation problem for diffusion processes.
For the sake of completeness, we review the basics of simulation and inference for SDEs in \cref{sec:inference_simulation}.
Finally, in \cref{sec:score_generative}, review score-based generative modeling, emphasizing its connections to the dynamic Schrödinger bridge problem.

\subsection{Dynamic, Static and Half Versions}\label{sec:schrodiger_problems}

For $Γ,ϒ ∈ 𝒫_1$ and $R ∈ 𝒫_𝒞$, the solution to the dynamic Schrödinger bridge problem $\DynSB$ is given by
\begin{equation*}
S^*(Γ,ϒ,R,𝒫_𝒞) ≔ \argmin_{S ∈ 𝒫_𝒞(Γ,ϒ)}D_{KL}(S \TO R).
\end{equation*}
$\DynSB$ seeks a stochastic process transporting an initial distribution $Γ$ to a terminal distribution $ϒ$, while achieving minimum Kullback-Leibler (KL) divergence to the reference law $R$.
For a given $S ∈ 𝒫_𝒞$, $S ≪ R$, the marginal-conditional decompositions for PMs and densities give
\begin{align}
D_{KL}(S \TO R) & = 𝔼_{S_{0,τ}}\Big[𝔼_{S_{•|0,τ}}\Big[\log\Big(\frac{dS_{0,τ}}{dR_{0,τ}}(x_0,x_τ)\frac{dS_{•|0,τ}}{dR_{•|0,τ}}(x)\Big)\Big]\Big] \notag \\
                & = D_{KL}(S_{0,τ} \TO R_{0,τ}) + 𝔼_{S_{0,τ}}[D_{KL}(S_{•|0,τ} \TO R_{•|0,τ})].\label{eq:kl_decomp_full}
\end{align}
The second term of \cref{eq:kl_decomp_full} is minimized by $S^*_{•|0,τ}(Γ,ϒ,R,𝒫_𝒞)=R_{•|0,τ}$, independently of $Γ,ϒ$, yielding the representation of $S^*(Γ,ϒ,R,𝒫_𝒞)$ as a mixture of the reference process pinned down at its initial and terminal values.
The mixing distribution $S^*_{0,τ}(Γ,ϒ,R,𝒫_𝒞)$ over the initial and terminal values minimizes the first term of \cref{eq:kl_decomp_full}, i.e.\ it solves a specific instance of the following problem.

For $Γ,ϒ ∈ 𝒫_1$ and $B ∈ 𝒫_2$, the solution to the static Schrödinger bridge problem $\StaSB$ is given by
\begin{equation*}
S^*(Γ,ϒ,B,𝒫_2) ≔ \argmin_{C ∈ 𝒫_2(Γ,ϒ)}D_{KL}(C \TO B).
\end{equation*}
$\StaSB$ differs from $\DynSB$ in the optimization space: $𝒫_2$ instead of $𝒫_𝒞$.
With this notation, $S^*_{0,τ}(Γ,ϒ,R,𝒫_𝒞) = S^*(Γ,ϒ,R_{0,τ},𝒫_2)$, the solution to $\StaSB$ for $B=R_{0,τ}$.
The solution to $\DynSB$ is then $S^*(Γ,ϒ,R,𝒫_𝒞) = S^*(Γ,ϒ,R_{0,τ},𝒫_2)R_{•|0,τ}$.
Vice versa, if $S^*(Γ,ϒ,R,𝒫_𝒞)$ solves $\DynSB$, then $S_{0,τ}^*(Γ,ϒ,R,𝒫_𝒞)$ solves $\StaSB$ for $B = R_{0,τ}$.
From this point onward, $\StaSB$ is always understood to be associated to the corresponding $\DynSB$, i.e.\ we will only consider $\StaSB$ for $B = R_{0,τ}$.
It established in \citet{ruschendorf1993note} that, if there is a $C ∈ 𝒫_2(Γ,ϒ)$ such that $D_{KL}(C \TO R_{0,τ}) < ∞$, then $\StaSB$ admits a solution, which is unique.
This hypothesis is assumed to hold thorough the paper.

We introduce an additional axis among which to classify Schrödinger bridge problems, covering both dynamic $\DynHSB$ and static $\StaHSB$ variants jointly for conciseness.
Thus, let $𝒫$ be either $𝒫_𝒞$ (case of $\DynHSB$) or $𝒫_2$ (case of $\StaHSB$).
For $Γ,ϒ ∈ 𝒫_1$, $Q ∈ 𝒫$, the solutions to the forward and to the backward Schrödinger half bridge problems are respectively given by
\begin{equation*}\begin{aligned}
S^*(Γ,□,Q,𝒫) & ≔ \argmin_{H ∈ 𝒫(Γ,□)}D_{KL}(H \TO Q), \\
S^*(□,ϒ,Q,𝒫) & ≔ \argmin_{H ∈ 𝒫(□,ϒ)}D_{KL}(H \TO Q).
\end{aligned}\end{equation*}
In contrast to $\DynSB$ and $\StaSB$, only one of the marginal conditions is enforced.
The forward and backward half bridge problems admit simpler solutions which form the basis of the developments of \cref{sec:ipf}.
Proceeding as in \cref{eq:kl_decomp_full}, it can be established that for $H ∈ 𝒫$ such that $H ≪ Q$
\begin{equation}\label{eq:kl_decomp_half}\begin{aligned}
D_{KL}(H \TO Q) & = D_{KL}(H_0 \TO Q_{0}) + 𝔼_{H_0}[D_{KL}(H_{•|0} \TO Q_{•|0})]  \\
                & = D_{KL}(H_τ \TO Q_{τ}) + 𝔼_{H_τ}[D_{KL}(H_{•|τ} \TO Q_{•|τ})],
\end{aligned}\end{equation}
from which we obtain the half bridge solutions $S^*(Γ,□,Q,𝒫) = Γ Q_{•|0}$ and $S^*(□,ϒ,Q,𝒫) = ϒ Q_{•|τ}$.
That is, one of the endpoint (or marginal, for $\StaHSB$) distributions of $Q = Q_0 Q_{•|0} = Q_τ Q_{•|τ}$ is replaced by one of $Γ,ϒ$ while keeping the associated conditional distribution of $Q$ constant.
In the case of $\DynHSB$, it thus suffices to propagate the initial (terminal) distribution $Γ$ ($ϒ$) through the forward (backward) dynamics of $Q$: $Q_{•|0}$ ($Q_{•|τ}$).

$R_0$ plays a minor role in $\DynSB$.
From \cref{eq:kl_decomp_half} applied with $Q = R$, $H ∈ 𝒫_𝒞$, either $R_0$ is such that $\DynSB$ do not admit any solution (when $D_{KL}(S \TO R) = +∞$ for every $S ∈ 𝒫_𝒞(Γ,ϒ)$), or the solution to $\DynSB$ exists, is unique, and is independent of $R_0$.
This work is concerned with the case where $R$ is given by the solution to a $d$-dimensional SDE, i.e.\ a diffusion process \citep{oksendal2003stochastic,friedman1975stochastic}.
We consider the following reference SDE, with initial distribution $Γ$ (without loss of generality),
\begin{equation}\label{eq:sde_ref}\begin{aligned}
 & dX_t = μ_R(X_t,t)dt + σ_R(X_t,t)dW_t,\quad t ∈ [0,τ], \\
 & X_0 ∼ Γ.
\end{aligned}\end{equation}
Thorough the paper, it is assumed that \cref{eq:sde_ref} admits a unique (weak, non-explosive) solution, whose law defines the reference law $R = Γ R_{•|0}$.
It is also assumed that $R ≪ R^{\circ}$ where $R^{\circ}$ is the law of the unique solution to \cref{eq:sde_ref} with $μ_R(x,t) = 0$, and that $Σ_R(x,t) ≔ σ_R(x,t)σ_R(x,t)ᵀ$ is invertible.
These are mild assumptions.
The law $R_{•|0}$ only depends on the drift coefficient $f(x,t)$ and on the diffusion coefficient $σ_R(x,t)$.

In concluding this section, it is noteworthy that solving $\StaSB$ is equivalent to solving a corresponding EOT problem.
We seek a solution $S^*(Γ,ϒ,R_{0,τ},𝒫_2)$, and for the sake of simplicity we examine the case where $Γ,ϒ,R_{0,τ}$ and $C ∈ 𝒫_2(Γ,ϒ)$ all admits densities with respect to Lebesgue measures.
Therefore,
\begin{equation*}
S^*(Γ,ϒ,R_{0,τ},𝒫_2) = \argmin_{C ∈ 𝒫_2(Γ,ϒ)} 𝔼_C[-\log r_{τ|0}(X_τ|X_0)] + H(C).
\end{equation*}
Here $H(C) ≔ 𝔼_C[\log c(X_0,X_τ)]$ is the entropy of $C$.
Given the same assumption, we consider a cost function $κ(x,y): ℝ^d × ℝ^d → ℝ$ and two target marginal distributions $Γ,ϒ$ of interest.
The solution to the EOT problem with regularization $ε > 0$ is then defined as
\begin{equation}\label{eq:eot}
EOT^*(Γ,ϒ,κ,ε) ≔ \argmin_{C ∈ 𝒫_2(Γ,ϒ)} 𝔼_C[κ(X_0,X_τ)] + εH(C).
\end{equation}
By judiciously selecting the reference measure $R$, it becomes feasible to solve a corresponding EOT problem of interest.
For example, if $R$ is associated to $dX_t = σdW_t$ for a scalar $σ > 0$, then $S^*(Γ,ϒ,R_{0,τ},𝒫_2) = EOT^*(Γ,ϒ,‖x - y‖^2,2σ^2)$.
Consequently, $S^*(Γ,ϒ,R_{0,τ},𝒫_2)$ is the solution to the Euclidean EOT problem for regularization $ε = 2σ^2$.
The arguments presented here carry over to the more generic measure-theoretic setting \citep[Remark 4.2]{peyre2020computational}.

\subsection{Iterative Proportional Fitting (IPF)}\label{sec:ipf}

The main tool employed in the numerical solution to $\DynSB$ and $\StaSB$ is an iterative procedure, known in the literature under a variety of names.
In $\StaSB$, when $Γ$ and $ϒ$ admit Lebesgue densities $γ(□)$ and $υ(□)$, the term \emph{Fortet iterations} \citep{fortet1940resolution} is used.
Instead, when $Γ$ and $ϒ$ concentrate all mass on a finite set of values, the term \emph{iterative proportional fitting procedure} \citep{deming1940least} is used.
In the same setting, due to the equivalence between $\StaSB$ and EOT, the same procedure is also known as \emph{Sinkhorn algorithm}.
Finally, the term IPF can be used in the context of the measure-theoretic formulation of \citet{ruschendorf1995convergence}, including the dynamic version $\DynSB$ \citep{bernton2019schrodinger}.
In this work, we utilize the IPF term.

\noindent
\begin{minipage}[t]{.48\textwidth}
\begin{algorithm}[H]
\caption{IPF}\label{alg:ipf}
\begin{algorithmic}[1]
\Input{$Γ,ϒ,R_{•|0},n$}
\Output{$\{F^{(i)}\}_{i=1}^n$}
\State{$F^{(0)} ← Γ R_{•|0}$}
\For{$i = 1,…,n$}
\If{$i$ is even}
\State{$F^{(i)} ← Γ F^{(i - 1)}_{•|0}$} \Comment{forward IPF}
\Else{}
\State{$F^{(i)} ← ϒ F^{(i - 1)}_{•|τ}$} \Comment{backward IPF}
\EndIf{}
\EndFor{}
\end{algorithmic}
\end{algorithm}
\vspace{1em}
\end{minipage}

For $\StaSB$ or $\DynSB$, the IPF procedure (\cref{alg:ipf}) is obtained by iteratively solving the forward and backward half bridge problems starting from the reference measure $R$.
More precisely, let $𝒫$ be either $𝒫_𝒞$ or $𝒫_2$, $R ∈ 𝒫$, $Γ,ϒ ∈ 𝒫_1$, and assume that $R = Γ R_{•|0}$.
The IPF procedure is defined by the iterates $F^{(0)} = R = S^*(Γ,□,R,𝒫)$, $F^{(1)} = S^*(□,ϒ,F^{(0)},𝒫)$, $F^{(2)} = S^*(Γ,□,F^{(1)},𝒫)$, and so on.
Under appropriate conditions \citep{ruschendorf1995convergence}, the PMs $F^{(i)}$ associated with the IPF iterates of \cref{alg:ipf} converge in total variation (TV) metric and in KL divergence to $S^*(Γ,ϒ,R,𝒫)$ as $i → ∞$.
Note that our choice of indexing of the IPF iterations differs from the prior literature, where indexing starts from $1$.
In this work, indexing reflects the number of learning problems that needs to be solved in order to produce samples from the corresponding IPF iteration.
As simulation from \cref{eq:sde_ref} is in general feasible (\cref{sec:inference_simulation}), and can be carried out exactly for the specifications of interest (\cref{sec:sde_class}), no learning is required to produce samples from $F^{(0)}$.

For $\DynSB$, and specifically for the case of a reference diffusion law $R$, \citet{bortoli2021diffusion,vargas2021solving} introduce implementations of \cref{alg:ipf} that rely on a time reversal result for diffusion processes \citep{anderson1982reversetime}.
Let $X$ be the $d$-dimensional diffusion process with law $P ∈ 𝒫_𝒞$ arising as unique solution to the SDE
\begin{equation}\label{eq:sde_fwd}\begin{aligned}
 & dX_t = μ(X_t,t)dt + σ_R(X_t,t)dW_t,\quad t∈[0,τ], \\
 & X_0 ∼ P_0.
\end{aligned}\end{equation}
Under suitable conditions, the time reversed process $\RX_t ≔ X_{τ - t}$ is still a diffusion process, associated to the SDE
\begin{equation}\label{eq:sde_bwd}\begin{aligned}
 & d\RX_t = υ(\RX_t,t)dt + σ_R(\RX_t,𝔯)dW_t,\quad t ∈ [0,τ], \\
 & υ(x,t) ≔ -μ(x,𝔯) + ∇⋅Σ_R(x,𝔯) + Σ_R(x,𝔯)∇_x\log p_𝔯(x),   \\
 & \RX_0 ∼ P_τ.
\end{aligned}\end{equation}
In \cref{eq:sde_bwd} $𝔯 ≔ τ-t$, $p_t(x)$ is the Lebesgue density of the marginal distribution $P_t$, and the $d$-dimensional vector $∇⋅Σ_R(x,t)$ is defined as $[∇⋅Σ_R(x,t)]_i\ ≔ ∑_{j=1}^d ∇_{x_j}[Σ_R(x,t)]_{i,j}$, $1≤i≤d$.
We refer to \citet{haussmann1986time,millet1989integration} for the conditions required on \cref{eq:sde_fwd} for the time reversal \cref{eq:sde_bwd} to hold.

\subsection{Diffusion IPF as Iterative Simulation-Inference}\label{sec:ipf_diffusion_approach}

\citet{bortoli2021diffusion,vargas2021solving} propose sampling-based implementations of \cref{alg:ipf}.
At iteration $i ≥ 1$, samples are generated from $F^{(i - 1)}$ to construct an approximation to $F^{(i)}$.
This approximation is used in turn to generate samples from $F^{(i)}$, which form the input to iteration $i + 1$.

Let $F^{(i - 1)}$ be a diffusion associated to a SDE (which holds for $i = 1$, as $F^{(0)} = R$).
We show that $F^{(i)}$ is also a diffusion associated to a different SDE.\@
First, lines 4 and 6 of \cref{alg:ipf} are modified respectively to $F^{(i)} ← Γ \RF^{(i-1)}_{•|0}$ and $F^{(i)} ← ϒ \RF^{(i-1)}_{•|0}$.
With this choice, the path PMs associated to backward IPF iterations are defined over a reverse (relatively to the forward IPF iterations) timescale.
Denote the drift and diffusion coefficients associated to $F^{(i - 1)}$ respectively with $μ_F^{(i - 1)}(x,t)$ and $σ_R(x,t)$, so that $μ_F^{(0)}(x,t) = μ_R(x,t)$.
It will become clear shortly why the diffusion coefficient is independent of $i$.
If $i$ is even, we have $F^{(i)}_0 = Γ$, otherwise $F^{(i)}_0 = ϒ$.
Thus, for all $i ≥ 1$, $F^{(i)} = F^{(i)}_0 \RF^{(i-1)}_{•|0}$.
Crucially, from the time reversal result of \citet{anderson1982reversetime}, we know that $\RF^{(i - 1)}$ is the law of the diffusion associated to
\begin{equation}\label{eq:ipf_previous}\begin{aligned}
 & d\RX_t^{(i-1)} = μ_F^{(i)}(\RX^{(i-1)}_t,t)dt + σ_R(\RX^{(i-1)}_t,𝔯)dW_t,\quad t ∈ [0,τ], \\
 & μ_F^{(i)}(x,t) ≔ -μ_F^{(i-1)}(x,𝔯) + ∇⋅Σ_R(x,𝔯) + Σ_R(x,𝔯)∇_x\log f^{(i-1)}_𝔯(x),         \\
 & \RX^{(i-1)}_0 ∼ F_τ^{(i-1)}.
\end{aligned}\end{equation}
Moreover, $\RF^{(i - 1)} = F^{(i-1)}_τ \RF^{(i - 1)}_{•|0}$ and $F^{(i)} = F^{(i)}_0 \RF^{(i - 1)}_{•|0}$ differs only in the initial distribution: $F^{(i)}$ is associated to \cref{eq:ipf_previous} with initial distribution $F^{(i)}_0$.
An ideal implementation would thus iteratively compute the drift coefficient $υ^{(i)}(x,t)$ from $F^{(i-1)}$ for all $i ≥ 1$.

It is known that the convergence of the IPF iterations becomes problematic in the small noise regime \citep{dvurechensky2018computational}, i.e.\ for a vanishing diffusion coefficient $σ_R(x,t)$.
The aforementioned theoretical construction of the IPF iterations provides some insight.
Consider $F^{(0)}$, solution to \cref{eq:sde_ref}.
Under suitable conditions \citep[Chapter 11]{stroock2006multidimensional}, if $σ_R(x,t)$ vanishes, the solution to \cref{eq:sde_ref} converges in law to the solution to the random ODE
\begin{equation}\label{eq:ipf_ode}\begin{aligned}
 & dX_t = μ_R(X_t,t)dt,\quad t ∈ [0,τ], \\
 & X_0 ∼ Γ.
\end{aligned}\end{equation}
The time reversal of \cref{eq:ipf_ode} is
\begin{equation}\label{eq:ipf_previous_ode}\begin{aligned}
 & d\RX_t = -μ_R(\RX_t,𝔯)dt,\quad t ∈ [0,τ], \\
 & \RX_0 ∼ R_τ,
\end{aligned}\end{equation}
and $F^{(1)}$ is the solution to \cref{eq:ipf_previous_ode} with $\RX_0 ∼ ϒ$.
But $F^{(2)}$ is given once again by the solution to \cref{eq:ipf_ode}, and the IPF iterations fail to converge.
The key issue is the disappearance of the $x$-score $∇_x\log f^{(i-1)}_𝔯(x)$ factor from \cref{eq:ipf_previous}, through which the laws of the IPF iterations propagate.

In practice, approximators of $μ_F^{(i)}(x,t)$, $i≥1$, are required.
Relying directly on the functional form of the drift coefficients $μ_F^{(i)}(x,t)$ forms the basis of the approach of \cref{sec:score_generative}, which is limited to $i=1$.
Indeed, for $i > 1$ scalability issues arise, as the nested application of \cref{eq:ipf_previous} results in $i$ approximators being employed to match $μ_F^{(i)}(x,t)$ at iteration $i$.
Thus, \citet{bortoli2021diffusion,vargas2021solving} propose to directly infer the drift coefficients $μ_F^{(i)}(x,t)$, as $F^{(i)}$ is the law of a diffusion process with known diffusion for every $i ≥ 0$.
The approach relies on two observations.
First, generating a sample $\RX^{(i - 1)} ∼ \RF^{(i - 1)}$ is trivial if we have access to a sample $X^{(i - 1)} ∼ F^{(i - 1)}$, as $\RX^{(i - 1)}_t = X^{(i - 1)}_{τ - t}$.
Second, samples from $\RF^{(i - 1)}$ suffice to carry out inference for $F^{(i)}$.
Indeed, $\RF^{(i - 1)}$ and $F^{(i)}$ differing only in the initial distribution, share the same drift and diffusion coefficients.
These considerations suggest the following strategy:
(i) use samples from $\RF^{(i - 1)}$ to infer an approximator $α_F^{(i)}(x,t) ≈ μ_F^{(i)}(x,t)$;
(ii) construct the SDE
\begin{equation}\label{eq:sde_ipf}\begin{aligned}
 & dX^{(i)}_t = α_F^{(i)}(X^{(i)}_t,t)dt + σ_R(X^{(i)}_t,𝔯)dW_t,\quad t ∈ [0,τ], \\
 & X^{(i)}_0 ∼ F^{(i)}_0,
\end{aligned}\end{equation}
whose solution is approximately distributed as $F^{(i)}$.
Samples from \cref{eq:sde_ipf} are used in turn as input to iteration $i + 1$.
After $n$ iterations a sequence of inferred drift coefficients $α_F^{(1)}(x,t)$, $\dots$, $α_F^{(n)}(x,t)$ is obtained.

A number of approximations are involved in the aforementioned approach.
First, the simulation from each SDE \cref{eq:sde_ipf} results in a discretization error.
Second, the inferred drift coefficients $α_F^{(i)}(x,t)$, differing from their ideal counterparts, give raise to approximation errors.
These errors arise from the finite amount of data simulated from $\RF^{(i-1)}$ (Monte Carlo error), from the finite model capacity of an approximator $α_F^{(i)}(x,t) ≈ μ_F^{(i)}(x,t)$, and from local minima in the optimization required to carry out inference.
Both discretization and approximation errors can be well controlled by increasing the required computation effort.

However, at a more fundamental level, inference for $α_F^{(i)}(x,t)$ is based on samples from $\RF^{(i - 1)}$.
As such we can expect a good approximation of $α_F^{(i)}(x,t)$ over the regions the path space  $ℝ^d × [0,τ]$ with non-negligible probability under $\RF^{(i - 1)}$.
As noted, $\RF^{(i - 1)}$ differs from $F^{(i)}$ in its initial distribution, because in general $F^{(i)}_0 ≠ F^{(i - 1)}_τ$.
The equality holds at convergence ($i → ∞$) or when $F^{(0)}_τ = R_τ = ϒ$, in which case the solution to $\DynSB$ is trivial: $S^*(Γ,ϒ,R,𝒫) = R$.
Due to the mismatch between $\RF^{(i - 1)}$ and $F^{(i)}$, it is possible for the simulated solution to \cref{eq:sde_ipf} to explore regions of the path space with negligible probability under $\RF^{(i - 1)}$, where essentially no information has been available at inference time about the value of $μ_F^{(i)}(X_t,t)$.
In \cref{sec:application_mixture} we show that, far from being just a theoretical concern, the simulation-inference distribution mismatch can have a concrete detrimental effect.
Our proposal, introduced in \cref{sec:dbm}, does not suffer from the aforementioned issue.

In addition to the discussion in this Section, \citet{vargas2021solving} investigates error accumulation within Diffusion IPF (DIPF) approaches, as well as failure instances resulting from insufficient exploration and from the difficulty of bridging distant distributions.
These issues are explored quantitatively via empirical experimentation, and initial strides are made towards the establishment of a theoretical framework.
Furthermore, \citet{fernandes2022shooting} explores the challenge of ``prior forgetting'' encountered in DIPF procedures.
In this context, the prior refers to the reference process that only affects the first iteration of DIPF algorithms.
In contrast, in our proposed approach, the reference process directly affects every step of the procedure.

\subsection{Inference and Simulation for SDEs}\label{sec:inference_simulation}

For a given drift coefficient $μ(x,t)$, consider SDE \cref{eq:sde_fwd} with corresponding diffusion solution $P^μ$.
Sample paths can be generated with arbitrary accuracy using a variety of discretization schemes, the simplest being the Euler scheme.
Given a number of time steps $m ≥ 1$, corresponding to a time interval $Δt = τ/m$, a discretization $Y$ starting at $Y_0 ∼ P_0^μ$ is generated sequentially on the time grid $\{0,Δt,\dots,τ\}$ as
\begin{equation}\label{eq:euler}
Y_{t+Δt} = Y_t + μ(Y_t,t)Δt + σ_R(Y_t,t)(W_{t+Δt} - W_t),\quad (W_{t+Δt} - W_t) \overset{\iid}{∼} 𝒩_d(0,ΔtI).
\end{equation}
Convergence of discretization schemes can be assessed according to different metrics.
Strong, or path-wise, convergence requires $E[‖X_τ - Y_τ‖] → 0$ as $Δt → 0$, where $X ∼ P^μ$ and the same Browning motion $W$ drives both $X$ and its discretization $Y$.
More appropriate to our setting, weak convergence requires $|E[f(X_τ) - f(Y_τ)]| → 0$  as $Δt → 0$ for $f(x): ℝ^d → ℝ$ belonging to a class of test functions.
\citet{kloeden1992numerical} contains a thorough coverage of discretization schemes for SDEs and of their convergence properties.

Inference for diffusion processes is a rich research area with long historical developments.
We review only two basic approaches and refer to \citet{hurn2007seeing,kessler2012statistical,fuchs2013inference} and references therein for a broader overview.
The key difficulty in performing maximum likelihood inference is that transition densities are seldom analytically available outside of restrictive SDEs' families.
Discrete time series data can be observed over arbitrarily long time intervals; however, simple approximations (such as \cref{eq:euler}) are accurate only over short time intervals.
In this sense, our setting simplifies inference since an arbitrary amount of data can be simulated at arbitrarily high frequencies.
Moreover, only the drift coefficient needs to be inferred.

Under suitable conditions, the Cameron-Martin-Girsanov formula provides the density between $P^μ_{•|0}$ and $P^γ_{•|0}$ for another drift coefficient $γ(x,t)$ \citep[Chapter 7]{liptser1977statistics},
\begin{equation}\label{eq:girsanov}
\frac{dP^μ_{•|0}}{dP^γ_{•|0}}(x) = \exp\left\{∫_0^τ[(μ - γ)ᵀΣ_R^{-1}](x_t,t)dx_t - \frac{1}{2} ∫_0^τ [(μ - γ)ᵀΣ_R^{-1}(μ + γ)](x_t,t) dt \right\}.
\end{equation}
In the following, let $α(x,t)$ be an approximating function for the drift coefficient $μ(x,t)$.
Maximum likelihood inference for $μ$ can be implemented with a driftless SDE acting as dominating measure $R^\circ$:
\begin{equation}\label{eq:mle}
μ(x,t) = \argmax_{α(x,t)}𝕆_{\mathrm{MLE}}(α,P^μ,Σ_R), \quad 𝕆_{\mathrm{MLE}}(α,P^μ,Σ_R) ≔ 𝔼_{P^μ}\Big[\frac{dP^α_{•|0}}{dR^\circ_{•|0}}(X)\Big].
\end{equation}
The integrals in \cref{eq:girsanov} need to be discretized, but the discretization errors can be controlled by simulating data at increasing frequencies.

An alternative approach is to consider one of the defining properties of diffusions, i.e. \citep[Chapter 5.4]{friedman1975stochastic}
\begin{equation*}
μ(x,t) = \lim_{Δt→0^+} \frac{𝔼_{P^μ}[X_{t+Δt} - X_t|X_t=x]}{Δt} ≈ 𝔼_{P^μ}\Big[\frac{X_{t+Δt} - X_t}{Δt} \mathrel{\Big|} X_t=x\Big]
\end{equation*}
for suitably small $Δt$, hence

\begin{equation}\label{eq:drift}\begin{aligned}
 & μ(x,t) ≈ \argmin_{α(x,t)}𝕃_{\mathrm{DM}}(α,P^μ,Δt),                                                                        \\
 & 𝕃_{\mathrm{DM}}(α,P^μ,Δt) ≔ 𝔼_{t ∼ 𝒰(0,τ)}\Big[𝔼_{P^μ}\Big[\Big{‖}α(X_t,t) - \frac{X_{t+Δt} - X_t}{Δt}\Big{‖}^2\Big]\Big].
\end{aligned}\end{equation}

In the context of the Euler scheme, the estimator that is based on \cref{eq:drift} corresponds to the estimator that relies on \cref{eq:mle}, under the condition that \cref{eq:girsanov} undergoes a piece-wise constant discretization from the left.
A difference in their implementation manifests in how \cref{eq:mle} incorporates all discretized values for a simulated path into the computation, contrary to \cref{eq:drift}, which sub-samples the time component.
The former approach recovers the drift estimator derived in \citet{vargas2021solving}, as well as the drift matching estimator of \citet{bortoli2021diffusion} presented in Appendix E, up to a vanishing term as $Δt → 0$.
The remaining estimators presented in \citet{bortoli2021diffusion} target either $μ(Y_t,t)Δt$ or $Y_t + μ(Y_t,t)Δt$ of the Euler discretization \cref{eq:euler}, again taking into account all discretization steps of a given path.

\subsection{Score-based Generative Modeling (SGM)}\label{sec:score_generative}

The approach of \citet{song2021scorebased} is simpler, corresponding to the computation of $F^{(1)}$ in \cref{alg:idbm}, i.e.\ the time reversal of \cref{eq:sde_ref}.
SDE \cref{eq:sde_ref} is chosen to ensure approximate conditional independence of $X_τ$ from $X_0$, such that $X_τ ∼ R_τ ≈ ϒ$ for a simple distribution $ϒ$.
For a dataset of interest, the distribution of $X_0$ is a smoothed version of the training data distribution, i.e. $D_η ≔ \frac{1}{n}∑_{s=1}^n 𝒩_d(x_s;0, η^2I_d) ∈ 𝒫_1$ for a small scalar $η ≥ 0$ where $\{x_s\}_{s=1}^n$ are the $n$ $d$-dimensional samples.
$η = 0$ corresponds to the empirical data distribution: $D_0 = \frac{1}{n}∑_{s=1}^n δ_{x_s}$.
As observed in \cref{sec:ipf_diffusion_approach}, due to the (approximate) decoupling of $X_τ$ from $X_0$, $F^{(0)}$ approximately solves $\DynSB$, $F^{(0)} = R ≈ S^*(D_η,ϒ,R,𝒫_𝒞)$, and $F^{(1)}$ amounts to computing its time reversal.
This also implies that, for specifications of \cref{eq:sde_ref} such that $X_τ$ is almost independent of $X_0$, there is no advantage in solving $\DynSB$.

In \citet{song2021scorebased}, an inferential procedure is developed and carried out to compute the time reversal $\RX$ of \cref{eq:sde_ref}, and $\RX_τ$ is simulated to produce samples with a distribution close to $D_η$.
A shortcoming of this approach is that achieving the decoupling of $X_τ$ from $X_0$ requires a large effective integration time (\cref{sec:sde_class}), which makes the generation process either computationally intensive or inaccurate.
More precisely, the specification of \cref{eq:sde_ref} introduced in \citet{song2021scorebased} for the CIFAR-10 datasets amounts to simulating a simpler SDE on the time interval $[0,50^2]$.
Consequently, $1,000$ time steps, using an ad-hoc predictor-corrector discretization scheme, are employed for the simulation of $\RX_τ$ to ensure the samples' visual quality, as seen in \cref{sec:application_score_generative}.

Even though the inference techniques of \cref{sec:ipf_diffusion_approach} could also be applied to infer $\RX$, SGM approaches leverage the specific form of the time reversal drift coefficient of \cref{eq:ipf_previous}, with the aim of learning the $x$-score $∇_{x}\log r_t(x)$.
Inference is based on the minimization of a scalable version of the $x$-score matching objective \citep{hyvarinen2005estimation}, i.e.
\begin{equation*}
𝕃_{\mathrm{SSM}}(α,R,t) ≔ 𝔼_R[‖α(X_t,t) - ∇_{X_t}\log r_{t|0}(X_t|X_0)‖^2] = 𝔼_R[‖α(X_t,t) - ∇_{X_t}\log r_t(X_t)‖^2] + c,
\end{equation*}
where $α(x,t): ℝ^d×[0,τ]→ℝ^d$ is function approximating $∇_{x}\log r_t(x)$ and $c$ is independent of $α$.
The equality, which established in \citet{vincent2011connection}, holds thanks to the mixture representation of $R_{0,t}$.
While computing $∇_{x}\log r_t(x)$ has computational cost $𝒪(n)$, which is impractical for large datasets, computing $∇_{x}\log r_{t|0}(x|y)$ is $𝒪(1)$ with respect to $n$.
We provide a simpler derivation:
\begin{equation}\label{eq:sm_identity}\begin{aligned}
∇_{x_t}\log r_t(x_t) & = \frac{∇_{x_t}∫r_{t|0}(x_t|x_0)R_0(dx_0)}{r_t(x_t)} = ∫∇_{x_t}\log r_{t|0}(x_t|x_0) \frac{r_{t|0}(x_t|x_0)}{r_t(x_t)}R_0(dx_0) \\
                     & = 𝔼_R[∇_{X_t}\log r_{t|0}(X_t|X_0)|X_t=x_t],
\end{aligned}\end{equation}
and by the projection property of conditional expectations,
\begin{equation*}
∇_{x}\log r_t(x) = \argmin_{α(x,t)}𝔼_R[‖α(X_t,t) - ∇_{X_t}\log r_{t|0}(X_t|X_0)‖^2].
\end{equation*}
Both our derivation and the one of \citet{vincent2011connection} rely on an exchange of limits which trivially holds when $R_0 = D_η$.
In order to infer $∇_x\log r_t(x)$ over all $t ∈ [0,τ]$, the following objective is considered in \citet{song2021scorebased}
\begin{equation*}
𝕃_{\mathrm{SSM}}(α,R) = 𝔼_{t ∼ 𝒰(0,τ)}[λ_t 𝕃_{\mathrm{SSM}}(α,R,t)],
\end{equation*}
where $λ_t: (0,τ) → ℝ_{>0}$ is a time-dependent regularizer.
Indeed, the conditional $x$-scores $∇_x\log r_{t|0}(x|y)$ always diverge as $t → 0$, contrary to the $x$-score $∇_x\log r_t(x)$, whose behavior for $t → 0$ is governed by $R_0$. Therefore, it is sensible to normalize the orders of magnitude of $𝕃_{\mathrm{SSM}}(α,R,t)$ over the range of $t$.

Due to the identity \cref{eq:sm_identity}, the generative process can be represented as
\begin{equation}\label{eq:sde_time_reversal}\begin{aligned}
 & d\RX_t = μ_F^{(1)}(\RX_t,𝔯)dt + σ_R(\RX_t,𝔯)dW_t,\quad t ∈ [0,τ],                           \\
 & μ_F^{(1)}(x,t) = -μ_R(x,t) + ∇⋅Σ_R(x,t) + Σ_R(x,t) 𝔼_R[∇_{X_t}\log r_{t|0}(X_t|X_0)|X_t=x], \\
 & \RX_0 ∼ ϒ.
\end{aligned}\end{equation}
In \citet{song2021scorebased} a neural network $\alpha_θ(x,t)$ parametrized by $θ$ is employed as approximating function $α(x,t)$.
Having obtained an approximation $\alpha_θ(x,t) ≈ ∇_{x}\log r_t(x)$, generation is achieved by numerically integrating \cref{eq:sde_time_reversal}, with $\alpha_θ(x,t)$ in place of $𝔼_R[∇_{X_t}\log r_{t|0}(X_t|X_0)|X_t=x]$, to sample $\RX_τ$.

\section{Diffusion Bridge Mixture Transport}\label{sec:dbm}

In this section, we develop the IDBM procedure.
From \cref{sec:schrodiger_problems}, we know that the solution to $\DynSB$ admits the representation $S^*(Γ,ϒ,R,𝒫_𝒞) = S^*(Γ,ϒ,R_{0,τ},𝒫_2)R_{•|0,τ}$.
We are concerned with the case where $R$ corresponds to the diffusion process solution to \cref{eq:sde_ref}.
In \cref{sec:diffusion_bridges} we characterize $R_{•|0,τ}$ as the law of a diffusion bridge.
Considering the class of processes $CR_{•|0,τ}$, indexed by $C ∈ 𝒫_2(Γ,ϒ)$, offers a natural means of approximating $S^*(Γ,ϒ,R,𝒫_𝒞)$.
At the same time, it is advantageous to exploit the dynamic nature of $\DynSB$, rather than attempting to directly solve $\StaSB$, by constructing a sequence of diffusion approximations to $S^*(Γ,ϒ,R,𝒫_𝒞)$, as in \cref{sec:ipf_diffusion_approach}.
Processes formulated as $CR_{•|0,τ}$ are investigated in \citet{jamison1974reciprocal,jamison1975markov}.
In these studies (\cref{sec:reciprocal}), it is established that $CR_{•|0,τ}$ constitutes a diffusion process if and only if $C = S^*(Γ,ϒ,R_{0,τ},𝒫_2)$, which occurs at the optimum of $\DynSB$.
To address this central issue, we rely on a result that allows the construction of a diffusion process matching the marginal distributions of a mixture of diffusion processes (\cref{sec:mixture_matching}), with $CR{•|0,τ}$ being a special case.
The resulting transport is elaborated upon in \cref{sec:dbm_transports}, where we establish that its iterative application, the IDBM procedure, converges in law to $S^*(Γ,ϒ,R,𝒫_𝒞)$.
Additionally, we present suitable inference objectives that enable sampling-based implementations of the IDBM procedure.

\subsection{Diffusion Bridges}\label{sec:diffusion_bridges}

Consider \cref{eq:sde_ref} with initial value $x_0$, i.e. $R_0 = δ_{x_0}$.
Probabilistically conditioning \cref{eq:sde_ref} on hitting a terminal value $x_τ$ at time $τ$ results in the following SDE\footnote{It is a particular case of \cref{eq:schrodinger_sde}, obtained by Doob h-transform, for $Γ=δ_{x_0}$ and $ϒ=δ_{x_τ}$.} with initial value $x_0$ and terminal value $x_τ$ \citep[Theorem 7.11]{sarkka2019applied}
\begin{equation}\label{eq:bridge}\begin{aligned}
 & dX_t = b_R(X_t,t)dt + σ_R(X_t,t)dW_t,\quad t ∈ [0,τ], \\
 & b_R(x,t) ≔ μ_R(x,t) + Σ_R(x,t)∇_x\log r_{τ|t}(x_τ|x), \\
 & X_0 = x_0.
\end{aligned}\end{equation}
The drift adjustment $Σ_R(x,t)∇_x\log r_{τ|t}(x_τ|x)$ forces the process to hit $x_τ$ at time $τ$ and the diffusion process solving \cref{eq:bridge} is known as the diffusion bridge from $(x_0,0)$ to $(x_τ,τ)$.
Consistently with the adopted notation we write $R_{•|0,τ}$ for its law.

\subsection{Reciprocal Processes and Solution to $\DynSB$}\label{sec:reciprocal}

\citet{jamison1974reciprocal} studies the properties of reciprocal processes.
A $d$-dimensional process $X$ on $[0,τ]$ is reciprocal if $∀\ s,t: 0 ≤ s < t ≤ τ$,
\begin{equation*}
\Pr[A_{(s,t)^c}∩B_{(s,t)}|X_s,X_t] = \Pr[A_{(s,t)^c}|X_s,X_t] \Pr[B_{(s,t)}|X_s,X_t],
\end{equation*}
whenever $A_{(s,t)^c}$ belongs to the $σ$-algebra generated by $\{X_r: 0 ≤ r < s\}$ or $\{X_r: t < r ≤ τ\}$ and $B_{(s,t)}$ to the $σ$-algebra generated by $\{X_r: s < r < t\}$.
A Markov process is reciprocal, but the converse does not hold without further conditions.
In \citet{jamison1974reciprocal}, it is demonstrated that tying down a Markov process at its initial and terminal values and taking a mixture over such values results in a reciprocal process.
For a process obtained through this construction, \citet[Theorem 3.1]{jamison1974reciprocal} characterizes the cases in which the resulting process is not only reciprocal but also Markov.
Let $X ∼ Q ∈ 𝒫_𝒞$ be a Markov process, let $q_{t|s}$ denote the associated family of transition densities, here assumed to exist, be strictly positive and continuous, and let $C ∈ 𝒫_2(Γ,ϒ)$ for some $Γ,ϒ ∈ 𝒫_1$.
Then $C Q_{•|0,τ}$ is Markov if and only if $∃\ V_0,V_τ$ $σ$-finite positive measures over $ℝ^d$ such that $C ≪ V_0{⊗}V_τ$, for the product measure $V_0{⊗}V_τ$, with density
\begin{equation}\label{eq:schrodinger_rnd}
\frac{dC}{dV_0{⊗}V_τ}(x_0,x_τ) = q_{τ|0}(x_τ|x_0).
\end{equation}
In particular, if at least one of the marginal distributions of $C$, i.e. $Γ$, $ϒ$, concentrates all mass to a single point, \cref{eq:schrodinger_rnd} is satisfied and $C Q_{•|0,τ}$ is Markov.
Moreover, given $Γ,ϒ ∈ 𝒫_1$, there are unique $C ∈ 𝒫_2(Γ,ϒ)$ and $V_0,V_τ$ $σ$-finite positive measures such that \cref{eq:schrodinger_rnd} holds \citep[Theorem 3.2]{jamison1974reciprocal}.
Finally, $C$ of \cref{eq:schrodinger_rnd} equivalently solves $\StaSB$ for $B = Q_{0,τ}$, i.e. $C = S^*(Γ,ϒ,Q_{0,τ},𝒫_2)$.
Within this setting, \cref{eq:schrodinger_rnd} is often presented in the alternative form $(dC/dQ_{0,τ})(x_0,x_τ) = φ_0(x_0)φ_τ(x_τ)$, where $φ_0, φ_τ: ℝ^d → ℝ_{≥0}$ are the (Schrödinger) potentials \citep{ruschendorf1993note,pavon2018datadriven,bernton2019schrodinger}.

\citet{jamison1975markov} specializes these result to the case where $Q$ is the law of a diffusion process.
Thus, let $Q = R$, where $R$ is given by the solution to \cref{eq:sde_ref}.
Assuming \cref{eq:schrodinger_rnd}, it is shown that $C R_{•|0,τ}$, and hence the solution to $\DynSB$, is realized by the diffusion
\begin{equation}\begin{aligned}\label{eq:schrodinger_sde}
 & dX_t = [μ_R(X_t,t) + Σ_R(X_t,t) ∇_{X_t}\log h(X_t,t)]dt + σ_R(X_t,t)dW_t, \\
 & h(x,t) ≔ ∫r_{τ|t}(x_τ|x_t)V_τ(dx_τ),                                      \\
 & X_0 ∼ Γ,
\end{aligned}\end{equation}
i.e.\ by means of Doob h-transform \citep[Chapter IV.6.39]{rogers2000diffusions}.
Typically, \cref{eq:schrodinger_sde} is not directly applicable, $V_τ$ being analytically unavailable.

Finally, \citet{daipra1991stochastic} characterizes $h(x,t)$ entering the SDE's drift in \cref{eq:schrodinger_sde} as the solution to a stochastic optimal control problem.
More precisely, consider $R^u$ associated to
\begin{equation*}\begin{aligned}
 & dX_t = [μ_R(X_t,t) + u(X_t,t)]dt + σ_R(X_t,t)dW_t, \\
 & X_0 ∼ Γ,
\end{aligned}\end{equation*}
where $u(x, t)$ represents the Markov control.
Then, $u(x, t) = Σ_R(X_t,t) ∇_{X_t}\log h(X_t,t)$ minimizes
\begin{equation}\label{eq:sb_optimal_drift}
𝕃_\mathrm{OC}(u,R^u,Σ_R) ≔ 𝔼_{R^u}\Big[∫_0^τ ‖u(X_t,t)‖^2_{Σ_R^{-1}}dt\Big]
\end{equation}
for the weighted squared norm $‖x‖^2_{A} ≔ xᵀAx$, under the constraint that $X_τ ∼ ϒ$.
We refer to \citet{daipra1991stochastic} for the required conditions and precise statement of this result.

\subsection{Diffusion Mixture Matching}\label{sec:mixture_matching}

The following result establishes that a mixture of diffusion processes can be matched in terms of marginal distributions by a single diffusion process.
\cref{thm:mixture_matching} is established in \citet[Corollary 1.3]{brigo2002general} limitedly to finite mixtures and 1-dimensional diffusions.
The proof and required assumptions are deferred to \cref{app:proofs}.

\begin{restatable}[Diffusion mixture matching]{theorem}{resmixturematching}\label{thm:mixture_matching}
Consider the family of $d$-dimensional SDEs indexed by $λ ∈ Λ$
\begin{equation}\label{eq:sde_family}\begin{aligned}
 & dX_t^λ = \mu^λ(X_t^λ,t)dt + σ^λ(X_t^λ,t)dW_t^λ,\quad t ∈ [0,τ], \\
 & X_0^λ \sim P_0^λ,
\end{aligned}\end{equation}
corresponding to the family of path PMs $\{P^λ\}_{λ∈Λ}$.
For a mixing PM $Ψ$ on $Λ$, let $Π ∈ 𝒫_𝒞$ be obtained by taking the $Ψ$-mixture of \cref{eq:sde_family} over $λ ∈ Λ$.
In particular, define the mixture marginal densities $π_t$, $t ∈ (0,τ)$, and the mixture initial PM $Π_0$ by
\begin{equation}\label{eq:distribution_mixture}
π_t(x) ≔ ∫_Λ p_t^λ(x) Ψ(dλ),\quad Π_0(dx) ≔ ∫_Λ P_0^λ(dx) Ψ(dλ).
\end{equation}
Consider the $d$-dimensional SDE
\begin{equation}\label{eq:sde_mixture}\begin{aligned}
 & dX_t = \mu(X_t,t)dt + σ(X_t,t)dW_t,\quad t ∈ [0,τ],     \\
 & \mu(x, t) ≔ \frac{∫_Λ \mu^λ(x,t)p_t^λ(x)Ψ(dλ)}{π_t(x)}, \\
 & σ(x, t) ≔ \frac{∫_Λ σ^λ(x,t)p_t^λ(x)Ψ(dλ)}{π_t(x)},     \\
 & X_0 \sim Π_0,
\end{aligned}\end{equation}
with law $P$.
Then, under mild conditions, for all $t ∈ [0,τ]$ it holds that $P_t = Π_t$.
\end{restatable}

\subsection{Diffusion Bridge Mixture Transports}\label{sec:dbm_transports}

In \cref{sec:diffusion_bridges} we introduced diffusion bridges interpolating initial values $x_0$ to terminal values $x_τ$.
For $C ∈ 𝒫_2(Γ,ϒ)$, consider $Π(C,R_{•|0,τ}) ∈ 𝒫_𝒞(Γ,ϒ)$ given by $Π(C,R_{•|0,τ}) ≔ C R_{•|0,τ}$, i.e. the mixture of diffusion bridges \cref{eq:bridge} over $(X_0,X_τ) ∼ C$.
We apply \cref{thm:mixture_matching} to $Π(C,R_{•|0,τ}) $ with $λ = (x_0,x_τ)$, $Λ = ℝ^d × ℝ^d$, and mixing distribution $Ψ(dλ)=C(dx_0,dx_τ)$.
The resulting mixture matching SDE has diffusion coefficient $σ_R(x,t)$ and drift coefficient $μ_M(x,t) ≔ μ_R(x,t) + Σ_R(x,t)e(x,t)$ where
\begin{equation*}\begin{aligned}
 & e(x_t,t) = \frac{∫∇_{x_t}\log r_{τ|t}(x_τ|x_t)π_{t|0,τ}(x_t|x_0,x_τ)Π_{0,τ}(dx_0,dx_τ)}{∫π_{t|0,τ}(x_t|x_0,x_τ)Π_{0,τ}(dx_0,dx_τ)} \\
 & \quad = ∫∇_{x_t}\log r_{τ|t}(x_τ|x_t)\frac{π_{t|0,τ}(x_t|x_0,x_τ)}{π_t(x_t)}Π_{0,τ}(dx_0,dx_τ)                                     \\
 & \quad = ∫∇_{x_t}\log r_{τ|t}(x_τ|x_t)Π_{τ|t}(dx_τ|x_t)                                                                             \\
 & \quad = 𝔼_{Π}[∇_{X_t}\log r_{τ|t}(X_τ|X_t)|X_t=x].
\end{aligned}\end{equation*}
In conclusion, let $M(Π(C,R_{•|0,τ})) ∈ 𝒫_𝒞(Γ,ϒ)$ be associated to
\begin{equation}\label{eq:dbm_fwd}\begin{aligned}
 & dX_t = μ_M(X_t,t)dt + σ_R(X_t,t)dW_t,\quad t ∈ [0,τ],                     \\
 & μ_M(x,t) = μ_R(x,t) + Σ_R(x,t) 𝔼_{Π}[∇_{X_t}\log r_{τ|t}(X_τ|X_t)|X_t=x], \\
 & X_0 ∼ Γ.
\end{aligned}\end{equation}
Then $M(Π(C,R_{•|0,τ}))$ satisfies $M_t(Π(C,R_{•|0,τ})) = Π_t(C,R_{•|0,τ})$ for all $t ∈ [0,τ]$.
In particular $M(Π(C,R_{•|0,τ}))$ transports $Γ$ to $ϒ$.
We refer to $M(Π(C,R_{•|0,τ}))$ as the diffusion bridge mixture (DBM) transport based on $C$.

Before turning our attention to the iterated application of the DBM transport, we consider two additional related transports.
The first transport is simply a specialization of the DBM transport to the case of a degenerate initial distribution, $Γ=δ_{x_0}$, resulting in
\begin{equation}\label{eq:dbm_fwd_x0}\begin{aligned}
 & dX_t = μ_{M,x_0}(X_t,t)dt + σ_R(X_t,t)dW_t,\quad t ∈ [0,τ],                             \\
 & μ_{M,x_0}(x,t) = μ_R(x,t) + Σ_R(x,t) 𝔼_{Π}[∇_{X_t}\log r_{τ|t}(X_τ|X_t)|X_t=x,X_0=x_0], \\
 & X_0 = x_0.
\end{aligned}\end{equation}
In view of the results presented in \cref{sec:reciprocal}, $Π(δ_{x_0}{⊗}ϒ,R_{•|0,τ})$ is the law of a diffusion process (not only a reciprocal process, as in the general case), specifically $S^*(δ_{x_0},ϒ,R,𝒫_𝒞)$.
Moreover, it is easy to see that \cref{thm:mixture_matching} applied to a single diffusion yields that same diffusion, not only a diffusion with matching marginal distributions, i.e. $M(Π(δ_{x_0}{⊗}ϒ,R_{•|0,τ})) = Π(δ_{x_0}{⊗}ϒ,R_{•|0,τ})$.
Indeed, it can be verified by direct calculation that \cref{eq:dbm_fwd_x0} and \cref{eq:schrodinger_sde} share the same drift coefficient when $Γ = δ_{x_0}$.
In summary, when $Γ = δ_{x_0}$, the DBM transport solves a simple version of $\DynSB$.
The use of the resulting SDE, sometimes termed Schrödinger-Föllmer Sampler, has been extensively explored in the literature, in generative \citep{wang2021deep,ye2022first}, sampling \citep{tzen2019theoretical,vargas2022bayesian,huang2021schrodingerfollmer,ruzayqat2023unbiased}, and optimal control \citep{daipra1991stochastic,zhang2022path} contexts.
These arguments carry over to a more general setting.
Whenever the initial coupling $C ∈ 𝒫_2(Γ,ϒ)$ is optimal, i.e.\ solves $\StaSB$, the resulting mixture process $Π(C,R_{•|0,τ})$ is a diffusion process, and $M(Π(C,R_{•|0,τ})) = Π(C,R_{•|0,τ})$.
That is, the DBM transport preserves the optimal coupling.

The second transport involves constructing the DBM in the reversed time direction.
More precisely, the following two approaches equivalently yield the same additional transport: (i) considering as reference SDE the time reversal of \cref{eq:sde_ref} with initial distribution $ϒ$ and constructing the DBM transport based on $\RC ∈ 𝒫_2(ϒ,Γ)$ or; (ii) performing a time reversal of \cref{eq:dbm_fwd}.
For a given $C ∈ 𝒫_2(Γ,ϒ)$, we refer to the resulting transport as backward DBM (BDBM) transport based on $C$.
Its law, $\RM(Π(C,R_{•|0,τ})) ∈ 𝒫_𝒞(ϒ,Γ)$, is associated to
\begin{equation}\label{eq:dbm_bwd}\begin{aligned}
 & d\RX_t = υ_M(\RX_t,𝔯)dt + σ_R(\RX_t,𝔯)dW_t,\quad t ∈ [0,τ],                           \\
 & υ_M(x,t) = -μ_R(x,t) + ∇⋅Σ_R(x,t) + Σ_R(x,t) 𝔼_Π[∇_{X_t}\log r_{t|0}(X_t|X_0)|X_t=x], \\
 & \RX_0 ∼ ϒ.
\end{aligned}\end{equation}
Then $\RM(Π(C,R_{•|0,τ}))$ satisfies $\RM_t(Π(C,R_{•|0,τ})) = \RPi_t(C,R_{•|0,τ})$ for all $t ∈ [0,τ]$.
In particular $\RM(Π(C,R_{•|0,τ}))$ transports $ϒ$ to $Γ$.
Comparing \cref{eq:sde_time_reversal} with \cref{eq:dbm_bwd} reveals that the only difference between the SGM model and the corresponding BDBM transport is the measure with respect to which the expectation in the drift adjustment factor is taken: $Π$ instead of $R$.

\noindent
\begin{minipage}[t]{.48\textwidth}
\begin{algorithm}[H]
\caption{IDBM}\label{alg:idbm}
\begin{algorithmic}[1]
\Input{$Γ,ϒ,R_{•|0,τ},C^{(0)},n$}
\Output{$\{M^{(i)}\}_{i=1}^n$}
\For{$i = 1,…,n$}
\State{$Π^{(i)} ← Π(C^{(i-1)},R_{•|0,τ})$}
\State{$M^{(i)} ← M(Π^{(i)})$}
\State{$C^{(i)} ← M^{(i)}_{0,τ}$}
\EndFor%
\end{algorithmic}
\end{algorithm}
\vspace{1em}
\end{minipage}

We consider the iterated application of the DBM transport, specified by \cref{alg:idbm}.
Starting from an initial coupling $C^{(0)}$, the initial-terminal distribution of the DBM transport of each iteration $i$ is employed to construct the diffusion bridge mixture of iteration $i + 1$.
As the BDBM transport is the time reversal of the DBM transport, any iteration $i$ of \cref{alg:idbm} can be equivalently formulated on the reverse timescale, resulting in $\RPi^{(i)}$, $\RM^{(i)}$ and $\RC^{(i)}$.
In \cref{thm:idbm_convergence} we establish the convergence properties of \cref{alg:idbm}.

\begin{figure}
\centering
\includegraphics[width=0.55\linewidth]{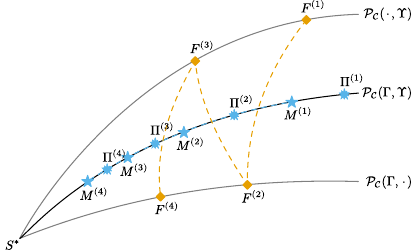}
\caption{Conceptual representation of IPF and IDBM iterations (inspired by Figure 1 of \citet{bernton2019schrodinger}), the ordering of $Π^{(i)}$ and $M^{(i)}$ is justified by \cref{thm:idbm_convergence}.}\label{fig:sketch}
\end{figure}

\begin{restatable}[IDBM convergence]{theorem}{residbmconvegence}\label{thm:idbm_convergence}
For $Γ,ϒ ∈ 𝒫_1$, $R ∈ 𝒫_𝒞$ associated to \cref{eq:sde_ref} with $σ_R(x,t) = I$, consider the iterates $Π^{(i)}, M^{(i)} ∈ 𝒫_𝒞(Γ,ϒ)$ of \cref{alg:idbm}.
Assume that\footnote{As in the proof, we denote $Π(C) ≔ Π(C,R_{•|0,τ})$, $M(C) ≔ M(Π(C,R_{•|0,τ}))$, $S^* ≔ S^*(Γ,ϒ,R,𝒫_𝒞)$.}: (i) $D_{KL}(C^{(0)} \TO S^*_{0,τ}) < ∞$; (ii) for each $i ≥ 1$ the Cameron-Martin-Girsanov theorem hold for $M^{(i)}$ yielding $dM^{(i)}/dS^*$ (implying that \cref{eq:dbm_fwd} for $Π = Π^{(i)}$ has a unique diffusion solution with law $M^{(i)}$, and that $M^{(i)} ≪ S^*$).
Then: (i) $Π^{(i)} \overset{ℒ}{⟶} S^*$ and $M^{(i)} \overset{ℒ}{⟶} S^*$ as $i → ∞$, where $\overset{ℒ}{⟶}$ denotes converge in law; (ii) $D_{KL}(Π^{(i)} \TO S^*) ≥ D_{KL}(M^{(i)} \TO S^*) ≥ D_{KL}(Π^{(i+1)} \TO S^*)$ for $i ≥ 1$ (implying that $D_{KL}(M^{(i)} \TO S^*)$ is non-increasing in $i$); (iii) $D_{KL}(Π(C) \TO S^*) = D_{KL}(M(C) \TO S^*)$ if and only if $Π(C) = M(C) = S^*$.
\end{restatable}

In generative modeling applications, the simplest choice is to set $C^{(0)}$ to the independent coupling $Γ{⊗}ϒ$, from which samples are trivially obtainable.
Dependent couplings are the natural choice in other applications.
In inverse problems, $C^{(0)}$ would be the distribution of pairs of clean and corrupted (or latent and partially observed) data.
Applications of the DBM transport to inverse problems and to aligned data are reviewed in \cref{sec:related_work}.
We briefly consider dependent initial couplings in \cref{sec:application_gaussian}.

We still need to address the problem of inferring the drift coefficients of \cref{eq:dbm_fwd} or \cref{eq:dbm_bwd} for each iteration $i$ of \cref{alg:idbm}.
The projection property of conditional expectations once again yields suitable objectives for the drift adjustments:
\begin{equation}\label{eq:dbm_drift_inference}\begin{aligned}
 & 𝔼_{Π}[Σ_R(X_t,t)∇_{X_t}\log r_{τ|t}(X_τ|X_t)|X_t=x] = \argmin_{α(x,t)}𝕃_{\mathrm{DBM}}(α,Π,R,Σ_R),             \\
 & 𝕃_{\mathrm{DBM}}(α,Π,R,Σ_R) ≔ 𝔼_{t ∼ 𝒰(0,τ)}[ϱ_t 𝔼_Π[‖α(X_t,t) - Σ_R(X_t,t)∇_{X_t}\log r_{τ|t}(X_τ|X_t)‖^2]],  \\
 & 𝔼_Π[Σ_R(X_t,t)∇_{X_t}\log r_{t|0}(X_t|X_0)|X_t=x] = \argmin_{α(x,t)}𝕃_{\mathrm{BDBM}}(α,Π,R,Σ_R),              \\
 & 𝕃_{\mathrm{BDBM}}(α,Π,R,Σ_R) ≔ 𝔼_{t ∼ 𝒰(0,τ)}[λ_t 𝔼_Π[‖α(X_t,t) - Σ_R(X_t,t)∇_{X_t}\log r_{t|0}(X_t|X_0)‖^2]].
\end{aligned}\end{equation}
In \cref{eq:dbm_drift_inference}, $ϱ_t: (0,τ) → ℝ_{>0}$ and $λ_t: (0,τ) → ℝ_{>0}$ are time-dependent regularizers that compensates for the diverging $∇_x\log r_{τ|t}(y|x)$ as $t → τ$ and $∇_y\log r_{t|0}(y|x)$ as $t → 0$.

Although we rely exclusively on \cref{eq:dbm_drift_inference} in the numerical experiments of \cref{sec:applications}, we can leverage additional inferential objectives.
Indeed, under mild conditions, performing inference following the description in \cref{sec:inference_simulation} on data simulated from $Π$ allows the recovery of the drift coefficients $μ_M(x,t)$ and $υ_M(x,t)$.
For instance, plugging-in the representation \cref{eq:dbm_fwd_x0} in \cref{eq:girsanov} yields $μ_M(x,t) = \argmax_{α(x,t)}𝕆_{\mathrm{MLE}}(α,Π,Σ_R)$, while that $μ_M(x,t) ≈ \argmin_{α(x,t)}𝕃_{\mathrm{DM}}(α,Π,Δt)$ follows directly from the definition of $Π$ by interchanging limits.
Proofs are omitted here.

We take a moment to provide several comments on the proposed IDBM procedure, contrasting it with the IPF procedure detailed in \cref{sec:ipf_diffusion_approach}:
\begin{enumerate}[label= (\roman*) ]
\item the IPF iterations generate a sequence of initial-terminal measures $F^{(i)}_{0,τ} ∈ 𝒫_2(Γ,□) ∪ 𝒫_2(□,ϒ)$ that matches one of the target marginal distributions $Γ$, $ϒ$ at a time, thus producing a valid coupling between $Γ$ and $ϒ$ only in the limit where $\StaSB$ is solved; in contrast, each IDBM iteration produces a valid initial-terminal coupling $C^{(i)} ∈ 𝒫_2(Γ,ϒ)$, with the sequence of coupling being progressively optimal toward solving $\StaSB$; this crucial difference is depicted in the sketch of \cref{fig:sketch};
\item IPF iterations necessarily alternates between forward and backward time directions; each IDBM iterations can be solved in either of the time direction (or both);
\item approaches relying on samples from $F^{(i-1)}$ to infer the drift coefficient of $F^{(i)}$ suffer from the simulation-inference distribution mismatch $\RF^{(i-1)} ≠ F^{(i)}$ which can negatively impact inferential efficiency; by construction $Π^{(i)}$ and $M^{(i)}$ share the same marginal distributions, so samples from $Π^{(i)}$ are guaranteed to cover regions of high probability according to $M^{(i)}$, thus making inference of the corresponding drift coefficient reliable;
\item the measures resulting from the IPF iterations do not admit a simple marginal-conditional decomposition; in contrast each $Π^{(i)}$ is given by a mixture of diffusion bridges.
\end{enumerate}
The aforementioned points suggest that using the IDBM can be advantageous within the context of generative models.
Point (i) establishes that truncating the IDBM iterations at a finite iteration number does not bias the transport.
Even the first iteration of the IDBM procedure produces a valid, even if suboptimal, transport.
In \cref{sec:application_score_generative} we follow this program as an alternative to the SGM approach.
In contrast, the IPF procedure requires employing many iterations, which can be computationally expensive.
Nevertheless, when multiple IDBM iterations are desirable, it might prove beneficial to solve the iterations alternating between the forward and backward time directions.
As the approximation and discretization errors cumulate over the iterates, repeatedly solving the IDBM procedure in the same time direction could lead to a terminal distribution progressively deviating from the corresponding target marginal distribution.
Points (iii) and (iv) point to potential efficiency gains of the IDBM procedure.
Indeed, at iteration $i > 1$, inferring the drift coefficients of both the IPF and the IDBM requires the numerical discretization and simulation from an SDE associated to iteration $i - 1$.
This is computationally expensive: in \citet{bortoli2021diffusion} paths are sampled, cached, and re-used for 100 steps of gradient descent.
With this choice, path sampling still amounts to roughly $50\%$ of the total computational time.
In the IDBM procedure, simulated paths, which are similarly costly to produce, give raise to samples from $C^{(i-1)}$.
However, this time, each sample from $C^{(i-1)}$ can be used to generate multiple samples at arbitrary time points from $Π^{(i)} = C^{(i-1)}R_{•|0,τ}$ which, for SDEs typically used in generative modeling (\cref{sec:sde_class}), can be done inexpensively and exactly (see \cref{eq:sde_class_birdge_td}).
We implement this approach in the empirical comparison of \cref{sec:application_matching}.

\section{Reference SDE Class}\label{sec:sde_class}

In this section we introduce a class of reference SDEs.
This is achieved in two steps: first we formulate a simple SDE, \cref{eq:sde_class_start}, then we show that SDEs commonly employed in generative models such as \citet{song2021scorebased}, i.e. SDE \cref{eq:sde_class}, are realized through a time change.

Consider the $d$-dimensional linear SDE
\begin{equation}\label{eq:sde_class_start}\begin{aligned}
 & dY_t = -α Y_t dt + Σ^{1/2} dW_t,\quad t ≥ 0, \\
 & Y_0 ∼ Γ,
\end{aligned}\end{equation}
with associated path PM $P$, where $α≥0$ is a scalar and $Σ$ is a positive definite covariance matrix.
For $α=0$, \cref{eq:sde_class_start} yields a correlated and scaled Brownian motion, otherwise \cref{eq:sde_class_start} corresponds to an Ornstein-Uhlenbeck process.
For $α=\frac{1}{2}$, \cref{eq:sde_class_start} has stationary distribution $𝒩_d(0,Σ)$.
The transition densities $p_{t|s}$ of \cref{eq:sde_class_start} are Gaussian:
\begin{equation}\label{eq:sde_class_a_v}\begin{aligned}
 & (Y_t|Y_s) ∼ 𝒩_d(Y_s a(s,t),Σ v(s,t)),                                    \\
 & a(s,t) = e^{-α(t - s)},                                                  \\
 & v(s,t) = \begin{cases*}
            t - s,                            & if $α = 0$, \\
            \frac{1}{2α}(1 - e^{-2α(t - s)}), & if $α > 0$,
            \end{cases*}
\end{aligned}\end{equation}
and
\begin{align}
 & ∇_{y_s}\log p_{t|s}(y_t|y_s) = Σ^{-1} \bigg(\frac{y_t}{a(s,t)} - y_s\bigg)\frac{a^2(s,t)}{v(s,t)},\label{eq:sde_class_x_score_s} \\
 & ∇_{y_t}\log p_{t|s}(y_t|y_s) = Σ^{-1} \bigg(y_s a(s,t) - y_t \bigg)\frac{1}{v(s,t)},\label{eq:sde_class_x_score_t}
\end{align}
which shows that the diffusion bridge \cref{eq:bridge} corresponding to \cref{eq:sde_class_start} remains a linear SDE.\@
From Bayes theorem and the Markov property, for $0 ≤ s < t < u$,
\begin{equation}\label{eq:sde_class_birdge_td}\begin{aligned}
 & (Y_t|Y_s,Y_u) ∼ 𝒩_d(Y_s \hat{a}(s,t,u) + Y_u \check{a}(s,t,u),\,Σ \tilde{v}(s,t,u)), \\
 & \hat{a}(s,t,u) = \frac{v(t,u)a(s,t)}{v(s,t)a^2(t,u)+v(t,u)},                         \\
 & \check{a}(s,t,u) = \frac{v(s,t)a(t,u)}{v(s,t)a^2(t,u)+v(t,u)},                       \\
 & \tilde{v}(s,t,u) = \frac{v(s,t)v(t,u)}{v(s,t)a^2(t,u)+v(t,u)}.
\end{aligned}\end{equation}

We consider the following time change of \cref{eq:sde_class_start}.
Let $β_t:[0,τ] → ℝ_{>0}$ be a continuous strictly positive function and $b_t ≔ ∫_0^t β_u du$.
Then $b_t:[0,τ] → [0,b_τ]$ is differentiable strictly increasing function and $β_t = \frac{db_t}{dt}$.
An application of \citep[Theorem 8.5.1]{oksendal2003stochastic} establishes that the time-changed process $X_t ≔ Y_{b_t}$ is equivalent in law to the solution to
\begin{equation}\label{eq:sde_class}\begin{aligned}
 & dX_t = -α β_t X_t dt + \sqrt{β_t}Σ^{1/2} dW_t,\quad t∈[0,τ], \\
 & X_0 ∼ Γ.
\end{aligned}\end{equation}
That is, SDE \cref{eq:sde_class} corresponds to the evolution of the simpler SDE \cref{eq:sde_class_start} under a non-linear time wrapping where time flows with instantaneous intensity $β_t$.

We conclude this section by specializing \cref{eq:sde_time_reversal,eq:dbm_fwd,eq:dbm_bwd} to the case of a reference SDE given by $dX_t = σdW_t$, $σ > 0$, referring to \cref{app:sde_class_extra} for the general setting.
The generative time reversal process \cref{eq:sde_time_reversal} is given by
\begin{equation*}\begin{aligned}
 & d\RX_t = \frac{𝔼_R[\RX_τ|\RX_t] - \RX_t}{τ - t}dt + σdW_t,\quad t ∈ [0,τ], \\
 & \RX_0 ∼ ϒ.
\end{aligned}\end{equation*}
The DBM transport \cref{eq:dbm_fwd} is given by
\begin{equation}\label{eq:sde_dbm_simple}\begin{aligned}
 & dX_t =  \frac{𝔼_{Π}[X_τ|X_t] - X_t}{τ - t}dt + σdW_t,\quad t ∈ [0,τ], \\
 & X_0 ∼ Γ,
\end{aligned}\end{equation}
while the BDBM transport \cref{eq:dbm_bwd} is given by
\begin{equation}\label{eq:sde_bdbm_simple}\begin{aligned}
 & d\RX_t = \frac{𝔼_{Π}[\RX_τ|\RX_t] - \RX_t}{τ - t}dt + σdW_t,\quad t ∈ [0,τ], \\
 & \RX_0 ∼ ϒ.
\end{aligned}\end{equation}
In this case the DBM and BDBM transports are symmetric: the BDM transport based on $C ∈ 𝒫_2(Γ,ϒ)$ is equivalent in law to the BDBM transport based on $\RC ∈ 𝒫_2(ϒ,Γ)$.

\section{Additional Related Literature}\label{sec:related_work}

The DBM transport is initially introduced in the unpublished manuscript \citep{peluchetti2021nondenoising}, which, however, lacks empirical validations.
This shortcoming is subsequently addressed by \citet{wu2022diffusionbaseda,liu2023learning}, who successfully implements the DBM transport across a multitude of applications, among other significant contributions.
\citet{liu2023learning} demonstrates that the DBM transport corresponds to an optimal Markovianization of the mixture of diffusion bridges it matches.
The theoretical developments associated with this finding contribute to our proof of \cref{thm:idbm_convergence}.
The proposal put forth by \citet{liu2023sb} is equivalent to the DBM transport for a reference scaled Brownian motion, and its application is shown to yield competitive results in image restoration problems.
Similarly, the proposal of \citet{somnath2023aligned} is equivalent to the DBM for the same reference dynamics.
It also assumes that the initial coupling is optimal, and thus preserved.
The resulting methodology is successfully applied to the case of aligned data.
While both works are equivalent to a specific instance of the DBM transport, they differ in certain aspects of implementation, such as the choice of the discretization scheme for sampling SDE paths and the definition of the loss regularizer.
These studies further substantiate the satisfactory performance we observe in the dataset transfer experiment of \cref{sec:application_matching} for the BDBM transport.

Concurrently and independently of our research, \citet{shi2023diffusion} introduces the \emph{Diffusion Schrödinger Bridge Matching---Iterative Markovian Fitting (DSBM-IMF)} procedure, which is equivalent to the IDBM procedure developed in this work.
\citet{shi2023diffusion} conducts a preliminary theoretical assessment of the convergence properties of the \emph{DSBM-IMF} procedure and empirically compares it with the proposal of \citet{bortoli2021diffusion}.
One of the numerical experiments pertains to the Gaussian setting detailed in \cref{sec:application_gaussian}, with $Σ_0 = Σ_1 = I$ and $d=50$.
In contrast to our investigation, neural networks approximators are utilized and trained iteratively.
The subsequent results are consistent with our findings of \cref{sec:application_gaussian}.

The Rectified Flow (RF) method, introduced in \citet{liu2022flow} and further explored in \citet{liu2022rectified}, is particularly pertinent to the IDBM procedure discussed in this paper due to their significant similarities.
The RF procedure commences with an initial coupling $C ∈ 𝒫_2(Γ,ϒ)$, from which it constructs a mixing process, denoted as $Φ ≔ CL_{•|0,1}$, via the deterministic linear interpolant process $L_{•|0,1}$.
The latter is defined by $X_t ≔ (1 - t)X_0 + tX_1$, where $t ∈ [0,1]$.
A rectification of this mixing process is then introduced, given by the solution to the RF ODE
\begin{equation*}\begin{aligned}
 & dV_t = ν(V_t,t)dt,\quad t ∈[0,1], \\
 & V_0 ∼ Γ,
\end{aligned}\end{equation*}
where $ν(x,t) ≔ 𝔼_Φ[X_1 - X_0 | X_t = x]$.

It is shown in \citet{liu2022flow} that this rectification procedure results in a valid coupling: $\mathrm{Law}(V_0,V_1) ∈ 𝒫_2(Γ,ϒ)$.
Moreover, when the rectification process is iterated, it yields a sequence of couplings, $\mathrm{Law}(V^{(i)}_0,V^{(i)}_1)$, where $i ≥ 1$, such that $𝔼[κ(V^{(i)}_1 - V^{(i)}_0)]$ is non-increasing across all convex cost functions $κ: ℝ^d → ℝ$.
For additional properties of the RF procedure, we direct the reader to \citet{liu2022flow}.
Finally, \citet{liu2022rectified} establishes that while the RF method successfully solves the one-dimensional Euclidean OT problem, it does not address the multidimensional variant.
Consequently, a modification of the RF approach is introduced, which is proven to effectively solve the multidimensional Euclidean OT problem.

The RF method can be understood as the limiting case of the IDBM procedure for a reference scaled Brownian motion as its randomness vanishes.
Indeed, consider the reference measure $R$ associated to $dX_t = σdW_t$ over $t ∈ [0,1]$.
$X_t ∼ R_{t|0,1}$ is realized by $X_t = (1 - t)X_0 + tX_1 + σ \sqrt{t(1 - t)}Z_t$, where $Z_t ∼ N(0, 1)$.
The DBM drift is given by
\begin{equation*}
μ_M(x,t) = 𝔼_Π\Big[\frac{X_1 - X_t}{1 - t} \mathrel{\Big|} X_t = x\Big] = 𝔼_Π\Big[X_1 - X_0 - σ\sqrt{\frac{t}{1 - t}}Z_t \mathrel{\Big|} X_t = x\Big],
\end{equation*}
for $Π = CR_{•|0,1}$.
As $σ → 0$ the Brownian bridge's marginal variances converge to zero, while its marginal means are independent of $σ$.
Informally, $Π → Φ$ and we obtain the aforementioned limiting equivalence.
\cref{sec:application_gaussian} expands on this connection within a fully Gaussian setting.
Moreover, the conclusions drawn from the numerical experiment of \cref{sec:application_matching} indicate that the optimal value of $σ$ depends on the specific application.
In particular, our attempts to apply the RF procedure to this dataset transfer experiment have not yielded satisfactory results.

As shown in \citet{shi2023diffusion}, under certain conditions, both the Flow Matching (FM) proposed by \citet{lipman2023flow} and the Conditional Flow Matching (CFM) suggested by \citet{tong2023conditional} correspond to the first iteration of the RF for generative modeling with a standard Gaussian as the initial distribution.
The OT-CFM variant, as introduced by \citet{tong2023conditional}, initially attempts to solve the EOT problem to derive an optimal coupling, followed by fitting a stochastic process to preserve this optimal coupling.
This approach is analogous to computing the DBM transport starting from the optimal coupling.

Finally, \citet{chen2022likelihood}, while sharing similarities with the work of \citet{bortoli2021diffusion,vargas2021solving}, introduces divergence-based objectives, while \citet{thornton2022riemannian} extends the scope of the DIPF method of \citep{bortoli2021diffusion} by generalizing it to non-Euclidean settings.

\section{Applications}\label{sec:applications}

In this section we consider four applications of the IDBM procedure.

\subsection{Gaussian Transports}\label{sec:application_gaussian}

We investigate in depth the case where both the initial and terminal distributions are $d$-dimensional Gaussian distributions, $Γ = 𝒩_d(μ_0, Σ_0)$, $ϒ = 𝒩_d(μ_1, Σ_1)$, and the reference measure $R$ is associated to $dX_t = σdW_t$ over the time interval $[0,1]$.
The solution to the Euclidean EOT problem, or to $\StaSB$, defined by $(Γ,ϒ,R)$, is the Gaussian coupling
\begin{equation}\label{eq:normal_rot}
\renewcommand\arraystretch{1.3}
S^*_{0,1}(Γ,ϒ,R,𝒫_𝒞) = 𝒩_{2d}\left(\begin{bmatrix}
μ_0 \\ μ_1
\end{bmatrix},\begin{bmatrix}
Σ_0     & Σ_S(σ) \\
Σ_S(σ)ᵀ & Σ_1
\end{bmatrix}\right),
\end{equation}
where $Σ_S(σ) ≔ (Σ_0Σ_1 + \frac{σ^4}{4}I)^{1/2} - \frac{σ^2}{2}I$.
The solution $OT^*(Γ,ϒ)$ to the OT problem is obtained setting $σ = 0$ in \cref{eq:normal_rot}\footnote{Gaussian distributions with positive semi-definite but not positive definite covariance matrices are well-defined through their characteristic function.}, with corresponding OT plan $φ_{OT}(x) ≔ μ_1 + Σ_0^{-1}Σ_S(0)(x - μ_0)$.
See for instance \citet[Section 2.6]{peyre2020computational} and \citet[Section 2]{janati2020entropic}.

Consider the DBM transport based on the Gaussian coupling $C ∈ 𝒫_2(Γ,ϒ)$,
\begin{equation}\label{eq:normal_dbm_c}
\renewcommand\arraystretch{1.3}
C = 𝒩_{2d}\left(\begin{bmatrix}
μ_0 \\ μ_1
\end{bmatrix},\begin{bmatrix}
Σ_0  & Σ_C \\
Σ_Cᵀ & Σ_1
\end{bmatrix}\right).
\end{equation}
The mixture of diffusion bridges with law $Π(C,R_{•|0,1})$ has a joint distribution $Π_{0,t,1}(C,R_{•|0,1}) $ which is again Gaussian,
\begin{equation*}
\renewcommand\arraystretch{1.3}
Π_{0,t,1}(C,R_{•|0,1}) = 𝒩_{3d}\left(\begin{bmatrix}
μ_0 \\ (1 - t)μ_0 + t μ_1 \\ μ_1
\end{bmatrix},\begin{bmatrix}
Σ_0        & Σ_{Π;0,t}  & Σ_C       \\
Σ_{Π;0,t}ᵀ & Σ_{Π;t,t}  & Σ_{Π;t,1} \\
Σ_Cᵀ       & Σ_{Π;t,1}ᵀ & Σ_1
\end{bmatrix}\right),
\end{equation*}
where $Σ_{Π;0,t} ≔ (1 - t)Σ_0 + tΣ_C$, $Σ_{Π;t,t} ≔ (1 - t)^2Σ_0 + t^2Σ_1 + t(1 - t)(Σ_C + Σ_Cᵀ + σ^2I)$ and $Σ_{Π;t,1} ≔ (1 - t)Σ_C + t Σ_1$.
It follows that
\begin{equation*}
𝔼_Π[X_1|X_t] = μ_1 + Σ_{Π;t,1}ᵀ Σ_{Π;t,t}^{-1}(X_t - μ_t)
\end{equation*}
and thus the DBM transport with law $M(Π(C,R_{•|0,1}))$ is given by the solution to
\begin{equation}\label{eq:normal_dbm_sde}\begin{aligned}
 & dX_t =  \frac{𝔼_{Π}[X_1|X_t] - X_t}{1 - t}dt + σdW_t,\quad t ∈ [0,1], \\
 & X_0 ∼ Γ.
\end{aligned}\end{equation}
SDE \cref{eq:normal_dbm_sde} is of the form $dX_t = (A_t X_t + b_t)dt + σdW_t$ for $A_t :[0,1] → ℝ^{d × d}$, $b_t: [0,1] → ℝ^d$, i.e.\ it is linear and time-inhomogenous with Gaussian transition probabilities.
The following representation holds: $X_1 | X_0$ is distributed as $P_1 X_0 + ε$, where $P_t$ is given by the solution to the matrix-valued ODE
\begin{equation}\label{eq:normal_dbm_fundamental}\begin{aligned}
 & dP_t = A_t P_t,\quad t ∈ [0,1], \\
 & P_0 = I,
\end{aligned}\end{equation}
and $ε$ is a $d$-dimensional Gaussian distribution whose parameters depend on the functions $A_t$ and $b_t$, but not on $X_0$.
In conclusion,
\begin{equation}\label{eq:normal_dbm_cp}
\renewcommand\arraystretch{1.3}
M_{0,1}(Π(C,R_{•|0,1})) = 𝒩_{2d}\left(\begin{bmatrix}
μ_0 \\ μ_1
\end{bmatrix},\begin{bmatrix}
Σ_0     & Σ_{C'} \\
Σ_{C'}ᵀ & Σ_1
\end{bmatrix}\right),\quad Σ_{C'} ≔ Σ_0 P_1ᵀ.
\end{equation}
Starting from the Gaussian coupling $C^{(0)} = Γ{⊗}ϒ$, i.e. $Σ_{C^{(0)}} = 0I$, the IDBM procedure iteratively computes the updates $Σ_C → Σ_{C'}$ resulting in a sequence of Gaussian couplings $C^{(i)}$.

The IPF updates for $\StaSB$ can be computed analytically thanks to the following property of Gaussian distributions.
Let $X$ and $Y$ be $d$-dimensional random variables with $(X,Y) ∼ P_{X,Y}$, where
\begin{equation*}
\renewcommand\arraystretch{1.3}
P_{X,Y} ≔ 𝒩_{2d}\left(\begin{bmatrix}
μ_x \\ μ_y
\end{bmatrix},\begin{bmatrix}
Σ_{xx}  & Σ_{xy} \\
Σ_{xy}ᵀ & Σ_{yy}
\end{bmatrix}\right),
\end{equation*}
and denote with $P_{Y|X}$ the conditional distribution of $Y$ given $X$ under $P_{X,Y}$.
For $Q_X ≔ 𝒩_d(λ_x,Γ_{xx})$, we have
\begin{equation}\label{eq:normal_ipf_update}\begin{aligned}
\renewcommand\arraystretch{1.3}
 & Q_X P_{Y|X} = 𝒩_{2d}\left(\begin{bmatrix}
                             λ_x \\ λ_y
                             \end{bmatrix},\begin{bmatrix}
                                           Γ_{xx}  & Γ_{xy} \\
                                           Γ_{xy}ᵀ & Γ_{yy}
                                           \end{bmatrix}\right),                    \\
 & λ_y ≔ μ_y + Σ_{xy}ᵀ Σ_{xx}^{-1}(λ_x - μ_x),                                      \\
 & Γ_{xy} ≔ Γ_{xx} Σ_{xx}^{-1} Σ_{xy},                                              \\
 & Γ_{yy} ≔ Σ_{yy} + Σ_{xy}ᵀ (Σ_{xx}^{-1} Γ_{xx} Σ_{xx}^{-1} - Σ_{xx}^{-1}) Σ_{xy}. \\
\end{aligned}\end{equation}
Therefore, starting from
\begin{equation*}
\renewcommand\arraystretch{1.3}
F^{(0)}_{0,1} = 𝒩_{2d}\left(\begin{bmatrix}
μ_0 \\ μ_0
\end{bmatrix},\begin{bmatrix}
Σ_0 & Σ_0        \\
Σ_0 & Σ_0 + σ^2I
\end{bmatrix}\right),
\end{equation*}
the IPF iterations update one of the marginal distributions of $F^{(i)}_{0,1}$ at a time via \cref{eq:normal_ipf_update}.

When $d=1$, ODE \cref{eq:normal_dbm_fundamental} can be solved analytically.
The coupling distribution \cref{eq:normal_dbm_c} has a single degree of freedom, its correlation coefficient, which we denote with $ρ_C$.
Then, $M_{0,1}(Π(C,R_{•|0,1})) ∈ 𝒫_2(Γ,ϒ)$ is Gaussian with correlation coefficient
\begin{equation}\label{eq:normal_dbm_corr}\begin{aligned}
 & ρ_M(ρ_C,Σ_0,Σ_1,σ) ≔ \exp\left\{-σ^2\frac{\tanh^{-1}(\frac{c_1}{c_3}) + \tanh^{-1}(\frac{c_2}{c_3})}{c_3}\right\} \\
 & c_1 = σ^2 + 2Σ_1(ρ_C Σ_0 - Σ_1), c_3 = \sqrt{(σ^2 + 2 (ρ_C + 1) Σ_0 Σ_1)(σ^2 + 2 (ρ_C - 1) Σ_0 Σ_1)},             \\
 & c_2 = σ^2 + 2Σ_0(ρ_C Σ_1 - Σ_0),
\end{aligned}\end{equation}
which is independent of $μ_0,μ_1$.
The following results hold: (i) for fixed $σ,Σ_0,Σ_1$, the map $ρ_M(ρ,Σ_0,Σ_1,σ) : ρ ↦ ρ'$ is a contraction over $ρ ∈ [-1,1]$ with limiting value $(\sqrt{Σ_0 Σ_1 + \frac{σ^4}{4}} - \frac{σ^2}{2})/\sqrt{Σ_0Σ_1}$, i.e.\ the correlation coefficient of $S^*_{0,1}(Γ,ϒ,R,𝒫_𝒞)$; (ii) for fixed $ρ,Σ_0,Σ_1$, $\lim_{σ->0}ρ_M(ρ,Σ_0,Σ_1,σ) = ρ_M(ρ,Σ_0,Σ_1,0) = 1$, i.e.\ the correlation coefficient of $OT^*(Γ,ϒ)$.
Result (i) establishes that the couplings produced by the IDBM iterations converge to $S^*_{0,1}(Γ,ϒ,R,𝒫_C)$, as expected from \cref{thm:idbm_convergence}.
Result (ii) shows that the DBM transport remains well-posed for vanishing regularization $σ$, and that $M^{(1)}_{0,1} \overset{ℒ}{⟶}OT^*(Γ,ϒ)$ as $σ→0$, i.e.\ convergence (in law) is attained by the first iteration of the IDBM procedure.
For further insight, reconsider the DBM SDE $dX_t = (A_t X_t + b_t)dt + σdW_t$.
In the noiseless limit $σ → 0$, the dynamics are specified by the DBM ODE $dx_t = (A^0_t x_t + b^0_t)dt$, which corresponds to the RF ODE (\cref{sec:related_work}), where $A^0_t ≔ \lim_{σ → 0}A_t$ and $b^0_t ≔ \lim_{σ → 0}b_t$.
Moreover, the solution $M$ to the DBM SDE converges in law to the solution to the DBM ODE \citep[Chapter 11]{stroock2006multidimensional}.
Solving the DBM ODE and computing the solution at terminal time yields $x_1 = μ_1 + \sqrt{Σ_1/Σ_0}(x_0 - μ_0) = φ_{OT}(x_0)$, the OT plan.

\vspace{1em}
\begin{center}
\begin{tabularx}{\textwidth}{*{2}{>{\centering\arraybackslash}X}}
\centering
\includegraphics[valign=m]{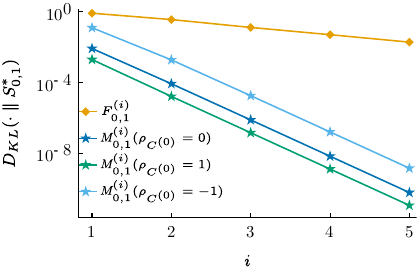}
 &
\resizebox{7.4cm}{!}{
\begin{tabular}{lll}
\toprule
$\lim$                & $D_{KL}(M^{(1)}_{0,1} \TO S^*_{0,1})$\footnote{$D_{KL}(M^{(1)}_{0,1} \TO S^*_{0,1})$ is independent of $μ_{\{0,1\}}$.} & $D_{KL}(F^{(1)}_{0,1} \TO S^*_{0,1})$ \\ \midrule
$μ_{\{0,1\}} → \pm ∞$ & $k$                                                                                                                    & $∞$                                   \\
$Σ_{\{0,1\}} → ∞$     & $K_Π$                                                                                                                  & $∞$                                   \\
$Σ_{\{0,1\}} → 0$     & $0$                                                                                                                    & $k$                                   \\
$σ → ∞$               & $0$                                                                                                                    & $0$                                   \\
$σ → 0$               & $k$                                                                                                                    & $∞$                                   \\ \bottomrule
\end{tabular}} \\
\captionof{figure}{Convergence of IPF and IDBM initial-terminal PMs to $S^*_{0,1}$ in KL divergence; $d=1$; log-scale.}\label{fig:normal_dbm_ipf_1d}
 &
\captionof{table}{Summary of limiting behavior of the first iteration of the IDBM and IPF procedures; $d=1$.}\label{tab:normal_dbm_cases_1d}
\end{tabularx}
\end{center}
\vspace{-1em}

We consider the scenario $μ_0=-1, μ_1=1, Σ_0=Σ_1=σ^2=1$.
We compute the KL divergence from the IPF PMs $F^{(i)}_{0,1}$ and from the IDBM couplings $M^{(i)}_{0,1}$ to $S^*_{0,1}(Γ,ϒ,R,𝒫_𝒞)$ as function of the iterations.
The KL divergence between multivariate Gaussian distributions can be calculated analytically, so no approximation is required.
The results are shown in \cref{fig:normal_dbm_ipf_1d} where, in addition to the standard initial independent coupling $C^{(0)} = Γ{⊗}ϒ$, we consider two additional initial couplings with correlations $ρ_{C^{(0)}} = \pm 1$.
It is observed that the IDBM iterations exhibit a higher rate of convergence with respect to the KL divergence.
We note that with our choice of indexing, we slightly favor the IPF iterations, as $F^{(0)}_{0,1}$ already incorporates the reference measure $R$, while $C^{(0)}$ does not depend on $R$.

\begin{figure}[H]
\centering
\includegraphics{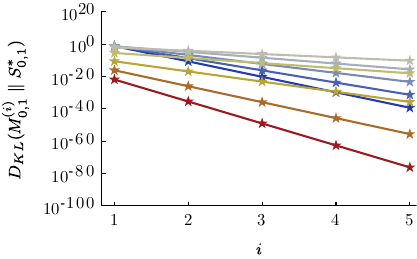}
\hfill
\includegraphics{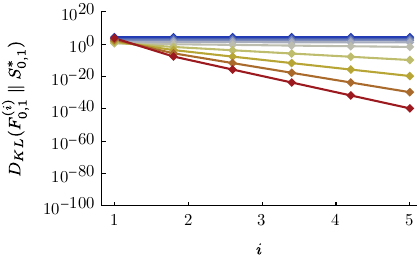}
\caption{Same as \cref{fig:normal_dbm_ipf_1d} for varying levels of regularization $σ$; (left): IDBM;\@ (right): IPF;\@ $d=1$; shared log-scale.}\label{fig:normal_dbm_ipf_sigma_1d}
\end{figure}

In \cref{fig:normal_dbm_ipf_sigma_1d} we modify the considered scenario by varying the level of regularization: $σ=10^s$ for $s=-2,…,2$, with $σ=10^{-2}$ corresponding to the blue color and $σ=10^2$ to the red color.
While the IDBM iterations converge quickly for both low and high levels of regularization (with $σ=1$, as in \cref{fig:normal_dbm_ipf_1d}, exhibiting the slowest convergence), the IPF iterations display increasing KL divergence for vanishing $σ$.

We complete the analysis of the one-dimensional setting by studying the behavior of $D_{KL}(M^{(1)}_{0,1} \TO S^*_{0,1})$ and $D_{KL}(F^{(1)}_{0,1} \TO S^*_{0,1})$ at the boundaries of the parameters' space spanned by $θ=(μ_0,μ_1,Σ_0,Σ_1,σ)$.
The results are reported in \cref{tab:normal_dbm_cases_1d}, where $k$ denotes positive finite values, $K_Π ≔ \frac{1}{4}(π + \log[4] -2(1 + \log[π]))$ is the supremum of $D_{KL}(M^{(1)}_{0,1} \TO S^*_{0,1})$ over $θ$, $μ_{\{0,1\}}$ stands for either of $μ_0$, $μ_1$, and $Σ_{\{0,1\}}$ stands for either of $Σ_0$, $Σ_1$.
When taking each limit all other parameters are kept constant.
Note that all of $M^{(1)}_{0,1}$, $F^{(1)}_{0,1}$ and $S^*_{0,1}$ depend on $θ$, even though the notation does not make this explicit.
The results of \cref{tab:normal_dbm_cases_1d} support the robustness of the IDBM procedure in the one-dimensional setting.
In particular, the IDBM procedure does not suffer from the difficulties in bridging distant distributions which are inherent to DIPF approaches \citep{vargas2021solving}.

\begin{figure}
\centering
\includegraphics{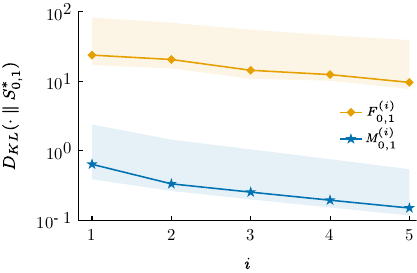}
\hfill
\includegraphics{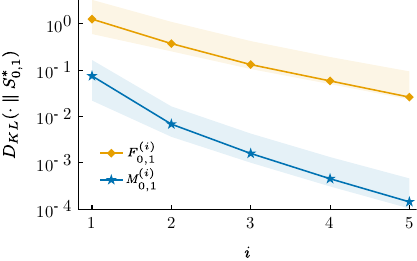}
\caption{Convergence of IPF and IDBM initial-terminal PMs to $S^*_{0,1}$ in KL divergence; (left): $d=5$; (right): $d=10$; log-scale.}\label{fig:normal_dbm_ipf_sigma_nd}
\end{figure}

For $d > 1$, we numerically integrate the matrix-valued ODE \cref{eq:normal_dbm_fundamental}.
We generate $20$ scenarios for each of $d=5,10$.
In each scenario, we sample $μ_{0},μ_1$ uniformly on $[-1,1]^d$ and $Σ_0,Σ_1$ from a Wishart distribution with $d$ degrees of freedom and covariance matrix $0.2I$ (as in \citet{janati2020entropic}).
We set $σ=0.2$ when $d=5$ and $σ=1$ when $d=10$.
In the multivariate instance, $M^{(1)}_{0,1}$ does not converge to $OT^*(Γ,ϒ)$ as $σ→0$, and the convergence rate of the IDBM procedure worsens as $σ$ approaches $0$.
Indeed, in the Gaussian setting defined for this study, the $σ→0$ limit of the DBM transport is equivalent to the RF of \citet{liu2022flow} (\cref{sec:related_work}), and \citet{liu2022rectified} establishes that the RF solves the OT problem only in the one-dimensional case.
Nonetheless, the IDBM procedure consistently outperforms the IPF procedure in all scenarios.
In \cref{fig:normal_dbm_ipf_sigma_nd} we plot the average KL divergences from $M^{(i)}_{0,1}$ and from $F^{(i)}_{0,1}$ to $S^*_{0,1}$ as function of the iterations with solid lines, where each average is over the $20$ sampled scenarios.
The banded regions correspond to the ranges between the highest and the lowest KL divergence values across the scenarios.

\subsection{Mixture Transports}\label{sec:application_mixture}

We consider $Γ = \frac{1}{3}(𝒩_1(-3, 0.2^2) + 𝒩_1(0.5, 0.2^2) + 𝒩_1(3, 0.2^2))$, $ϒ = 𝒩_1(0, 2^2)$, and $R$ associated to $dX_t = σdW_t$ over the time interval $[0,1]$ for $σ=0.2$.
This configuration is selected to illustrate the limitation of the sampling-based DIPF procedures discussed in \cref{sec:ipf_diffusion_approach}.

The goal of this application is to construct a transport from $ϒ$ to $Γ$.
This toy experiment can be seen as a simplified instance of the generative setting introduced in \cref{sec:score_generative}, where the training data consist of the values $-3,0.5,3$ and $Γ = D_η$ for $η=0.2$.
Compared with typical applications, $σ$ has been set to a small value to introduce a strong dependency between $X_0$ and $X_1$ under $R$, making SGM approaches inapplicable.

\begin{figure}
\centering
\includegraphics{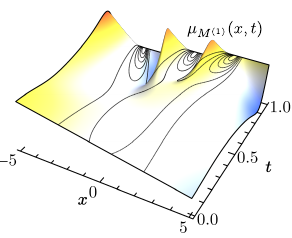}
\hfill
\includegraphics{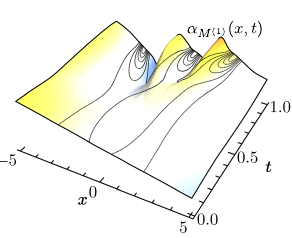}
\hfill
\includegraphics{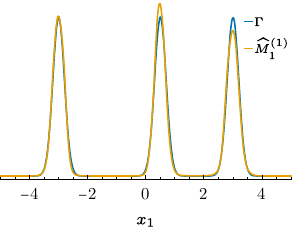}
\includegraphics{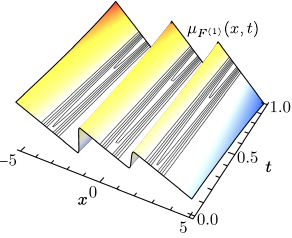}
\hfill
\includegraphics{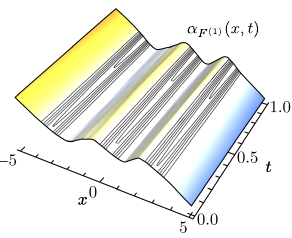}
\hfill
\includegraphics{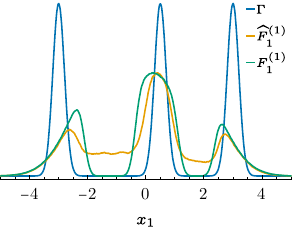}
\caption{Drift (left), estimated drift (center) and terminal density (right) resulting from the first iteration of the IPF (bottom) and IDBM (top) procedures; generative time.}\label{fig:mixture_drifts}
\end{figure}

In this application, we always assume time $t$ on the generative timescale.
The first iteration of the IPF procedure relies on samples from $R$ to infer $μ_{F^{(1)}}(x,t)$.
We display $μ_{F^{(1)}}(x,t)$ in \cref{fig:mixture_drifts} (bottom-left), superimposed with gray lines representing the level sets of the marginal densities $r_t$, $t ∈ [0,1]$.
Samples from $R$ have negligible probability of falling outside narrowly defined regions.
We employ as drift approximator a fully-connected neural network $α(θ,x,t)$ with 3 hidden layers, the ReLU ($x → \max(0, x)$) activation function, and a width of $512$ units.
Utilizing objective \cref{eq:drift}, $α(θ,x,t)$ is optimized via stochastic gradient descent (SGD).
We display the resulting inferred drift coefficient at convergence, $α_{F^{(1)}}(x,t)$, in \cref{fig:mixture_drifts} (bottom-center).
Again we superimpose with the level sets of $r_t$.
The approximation is accurate in regions of high probability under $R$, but significantly deteriorates outside these regions.
We contrast the true density of $F^{(1)}_1$ with a kernel density estimate based on 10 million samples from $\widehat{F}_1^{(1)}$, which are obtained by applying the Euler scheme, $Δt = 10^{-4}$, to
\begin{equation*}\begin{aligned}
 & dX_t = α_{F^{(1)}}(X_t,t)dt + σdW_t,\quad t ∈ [0,1], \\
 & X_0 ∼ ϒ,
\end{aligned}\end{equation*}
and collecting the terminal states.
As illustrated in \cref{fig:mixture_drifts} (bottom-right), the approximation errors have a significant impact.

Implementing sampling-based DIPF procedures requires carefully tuning the level of regularization.
On the one hand, an excessively high level of $σ$ renders IPF iterations after the first superfluous, with $F^{(1)}$ already approximately solving $\DynSB$.
On the other hand, an excessively low level of $σ$ risks incurring the difficulties just exposed.
We remark that, in particular instances, the simulation-inference mismatch can become irrelevant.
This is the case for the Gaussian transports of \cref{sec:application_gaussian}, where $F^{(i)}_t$ is a multivariate Gaussian distribution for each $i$ and $t$ \citep{mallasto2022entropyregularized}.
Therefore, $μ_{F^{(i)}}(x, t)$ is an affine function in $x$ for each $i$ and $t$, in which case local information about  $μ_{F^{(i)}}(x, t)$ is global as well.

We construct the IDBM procedure started from $C^{(0)} = ϒ{⊗}Γ$.
The drift coefficient $μ_{M^{(1)}}(x,t)$ corresponding to the first IDBM iteration is illustrated in \cref{fig:mixture_drifts} (top-left), superimposed with gray lines representing the level sets of the marginal densities $π^{(1)}_t$, $t ∈ [0,1]$.
Drift inference is based on samples from $Π^{(1)}$.
The same neural network architecture is employed, and utilizing objective \cref{eq:dbm_drift_inference} $α(θ,x,t)$ is optimized via SGD.\@
We display the resulting inferred drift coefficient at convergence $α_{M^{(1)}}(x,t)$ in \cref{fig:mixture_drifts} (top-center), superimposing the level sets of $π^{(1)}_t = m^{(1)}_t$.
As expected, $α_{M^{(1)}}(x,t)$ closely approximates $μ_{M^{(1)}}(x,t)$ in regions of high probability under $Π^{(1)}$.
Crucially, this suffices for learning $μ_{M^{(1)}}(x,t)$ on the regions of the path space relevant for the simulation from $M^{(1)}$.
An interesting observation is that the drift coefficient $μ_{M^{(1)}}(x,t)$ is smoother\footnote{Both drift coefficients are analytical functions.} than $μ_{F^{(1)}}(x,t)$.
This aspect makes it easier to both approximate $μ_{M^{(1)}}(x,t)$ with a neural network and accurately simulate from the resulting SDE.\@
We contrast the true density of $Γ$ with a kernel density estimate based on 10 million samples from $\widehat{M}_1^{(1)}$, which are obtained by applying the Euler scheme, $Δt = 10^{-4}$, to
\begin{equation*}\begin{aligned}
 & dX_t = α_{M^{(1)}}(X_t,t)dt + σdW_t,\quad t ∈ [0,1], \\
 & X_0 ∼ ϒ,
\end{aligned}\end{equation*}
and collecting the terminal states.
The close agreement between $Γ$ and $\widehat{M}_1^{(1)}$ is illustrated in \cref{fig:mixture_drifts} (top-right).

\begin{figure}
\centering
\includegraphics{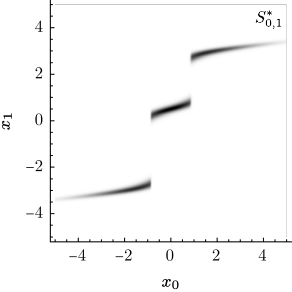}
\hfill
\includegraphics{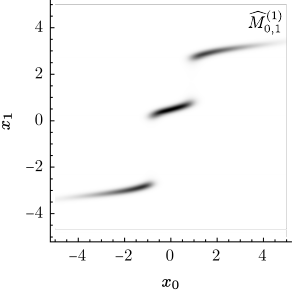}
\hfill
\includegraphics{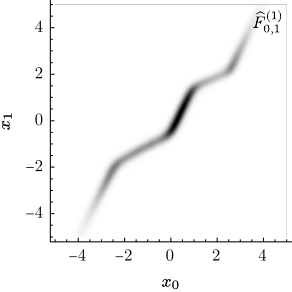}
\caption{Optimal coupling (left), coupling from the first iteration of the IDBM procedure (center), PM from the first iteration of the IPF procedure (right); generative time.}\label{fig:mixture_joint}
\end{figure}

We discretize $ϒ$ and $Γ$ into $5,000$ equally spaced bins on $[-5,5]$ and solve the corresponding EOT problem using the Sinkhorn algorithm from the POT library \citep{flamary2021pot}.
For $σ=0.2$, convergence is attained in around $2,000$ iterations with default tolerances.
We display the resulting 2D histogram in \cref{fig:mixture_joint} (left).
We additionally plot the 2D kernel density estimates of $\widehat{F}^{(1)}_{0,1}$ (center) and $\widehat{M}^{(1)}_{0,1}$ (right).
It can be seen that the first iteration of the IDBM procedure suffices to produce a coupling capturing the overall shape of $S^*_{0,1}$.

\subsection{Generative Modeling}\label{sec:application_score_generative}

In this application, we explore the use of the BDBM transport as an alternative to the score-based generative paradigm (\cref{sec:score_generative}) for image generation purposes.
The reference SDE is given by \cref{eq:sde_class} with $α=0$ and $Σ = I$, i.e.\ by a time-scaled $d$-dimensional Brownian motion.
We consider the CIFAR-10 dataset, hence $d = 28 {×} 28 {×} 3$.
As baseline modeling choice we consider the ``VE SDE'' parametrization from \citet{song2021scorebased} for $β_t$, and the corresponding smoothed training data distribution $D_η$,
\begin{equation}\label{eq:cifar_10_ve_sde}\begin{aligned}
 & dX_t = \sqrt{β_t}dW_t,\quad t ∈ [0,1],                                                             \\
 & X_0 ∼ D_η,                                                                                         \\
 & β_t ≔ σ_{\min}^2\Big(\frac{σ_{\max}}{σ_{\min}}\Big)^{2t} 2\log\Big(\frac{σ_{\max}}{σ_{\min}}\Big), \\
 & b_t = σ_{\min}^2\Big(\frac{σ_{\max}}{σ_{\min}}\Big)^{2t} - σ_{\min}^2,                             \\
\end{aligned}\end{equation}
where $η = σ_{\min} = 10^{-2}$ and $σ_{\max} = 50$.
We recall that $D_η ≔ \frac{1}{n}∑_{s=1}^n 𝒩_d(x_s;0, η^2I_d)$.
Here $x_1,…,x_n$ are the $50,000$ samples of CIFAR-10's training dataset $D^{\mathrm{train}}_{\mathrm{CIFAR-10}}$.
We select the NCSN++ (smaller) architecture, and training is performed using the official PyTorch implementation\footnote{\url{https://github.com/yang-song/score_sde_pytorch}.} of \citet{song2021scorebased}.
We review the SGM training procedure in \cref{alg:sgm_pseudo_code}.

\noindent
\begin{minipage}[t]{0.495\textwidth}
\begin{algorithm}[H]
\caption{SGM training}\label{alg:sgm_pseudo_code}
\begin{algorithmic}[1]
\Input{$Γ$, \hlred{$R_{t|0}$}, $∇_{y}\log r_{t|0}(y,x)$, $α(θ,x,t)$}
\Output{$α_{\mathrm{SGM}}(x,t)$}
\Repeat{}
\State{$t ∼ 𝒰(0,τ)$}
\State{$X_0 ∼ Γ$}
\State{}
\State{\hlred{$X_t ∼ R_{t|0}(□|X_0)$}}
\State{$Y_t ← ∇_{X_t}\log r_{t|0}(X_t|X_0)$}
\State{$ℒ ← \big\lVert Y_t - α(θ,X_t,t)\big\rVert^2 λ_t$}
\State{$θ ← \texttt{sgdstep}(θ,ℒ)$}
\Until{convergence}
\end{algorithmic}
\end{algorithm}
\end{minipage}%
\hfill
\begin{minipage}[t]{0.495\textwidth}
\begin{algorithm}[H]
\caption{BDBM training}\label{alg:bdbm_pseudo_code}
\begin{algorithmic}[1]
\Input{$Γ$, \hlblue{$ϒ$}, \hlblue{$R_{t|0,τ}$}, $∇_{y}\log r_{t|0}(y,x)$, $α(θ,x,t)$}
\Output{$α_{\mathrm{BDBM}}(x,t)$}
\Repeat{}
\State{$t ∼ 𝒰(0,τ)$}
\State{$X_0 ∼ Γ$}
\State{\hlblue{$X_τ ∼ ϒ$}}
\State{\hlblue{$X_t ∼ R_{t|0,τ}(□|X_0,X_τ)$}}
\State{$Y_t ← ∇_{X_t}\log r_{t|0}(X_t|X_0)$}
\State{$ℒ ← \big\lVert Y_t - α(θ,X_t,t)\big\rVert^2 λ_t$}
\State{$θ ← \texttt{sgdstep}(θ,ℒ)$}
\Until{convergence}
\end{algorithmic}
\end{algorithm}
\vspace{1em}
\end{minipage}

Implementing the BDBM approach requires minimal changes, as delineated in \cref{alg:bdbm_pseudo_code}.
PyTorch implementations of SGM and BDBM losses are provided for a reference SDE described by \cref{eq:sde_class} with $Σ = I$ and for the regularizer $λ_t = v(0, t)$ in \cref{code:sgm,code:bdbm} (\cref{app:score_generative}).

By design, the reference SDE \cref{eq:cifar_10_ve_sde} strongly decouples $X_1$ from $X_0$, in which case the BDBM and SGM approaches yield similar generative models.
We set $σ_{{\max}} = 1$ and keep $ϒ = 𝒩_d(0, σ_{{\max}}^2I)$ for the generative process's initial distribution, without making efforts to optimize these hyperparameters.
Aside from the modifications to the loss function, no additional changes are introduced to the training code, and training proceeds as for the baseline SGM.\@

The trained neural network approximators obtained from \cref{alg:sgm_pseudo_code} and \cref{alg:bdbm_pseudo_code} corresponds to the drift adjustment factors appearing in \cref{eq:sde_time_reversal} and \cref{eq:dbm_bwd} as conditional expectations.
As the remaining terms in \cref{eq:sde_time_reversal} and \cref{eq:dbm_bwd} are identical, no code changes are necessary when transitioning from SGM to BDBM at generation time.

To assess the visual quality of generated samples, we employ the Fréchet Inception Distance (FID, \citet{heusel2017gans}).
FID, which aligns with human-perceived visual quality, is calculated as the 2-Wasserstein Euclidean distance in features space, where image features are derived form the Inception model architecture.
CIFAR-10 training occurs over $1,300,000$ SGD steps.
After every $50,000$ SGD steps, we compute the FID between the test portion of the CIFAR-10 dataset, $D^{\mathrm{test}}_{\mathrm{CIFAR-10}}$, and the corresponding $5,000$ samples generated from the SGM and BDBM models.
Sampling is performed using the Euler scheme \cref{eq:euler}, $Δt = 10^{-3}$.
The results are presented in \cref{fig:generative_fid}, illustrating the competitive performance of the BDBM model for this discretization interval.
Furthermore, the BDBM approach exhibits significantly accelerated training.

\begin{figure}
\centering
\includegraphics{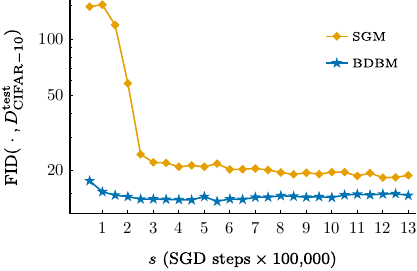}
\caption{FID between samples generated from the SGM and BDBM models and CIFAR-10 test empirical distribution at $Δt = 10^{-3}$; log-scale.}\label{fig:generative_fid}
\end{figure}

For each of the considered models and their 26 parameter states (checkpoints, at intervals of $50,000$ SGD steps), generating $5,000$ samples from either the SGM or the BDBM model using the Euler scheme with a discretization interval $Δt = 10^{-4}$ necessitates approximately one day of computation on the NVIDIA RTX 3090 GPU utilized in this experiment.
With the same discretization interval, the predictor-corrector scheme of \citet{song2021scorebased} demands twice the computation time.
As such, we have chosen to circumvent the intensive computations required to identify an ``optimal'' checkpoint through grid-search across parameter states using the predictor-corrector scheme and $Δt = 10^{-4}$, as is done in \citet{song2021scorebased}.
Our analysis is instead confined to the Euler scheme and a single ``optimal'' checkpoint is selected for both the SGM and BDBM models relying on the findings of \cref{fig:generative_fid}.
For the BDBM model, we opt for the model trained for $250,000$ steps, while for the SGM model we select the model trained for $1,200,000$ steps, which aligns with the checkpoint chosen in \citet{song2021scorebased}.
Subsequently, we calculate the FID values corresponding to $Δt = 1/25, 1/100, 1/1000$, obtaining $93.7, 14.0, 12.1$ for the BDBM model and $420.3, 18.3, 12.0$ for the SGM model, respectively.
Corresponding randomly selected samples are shown in \cref{fig:generative_grid} and contrasted with the first $16$ CIFAR-10 training samples.
Although the visual quality of the SGM and BDBM models is comparable at finer discretization steps, the BDBM model retains superior visual quality as the discretization interval increases.

\begin{figure}
\centering
\includegraphics{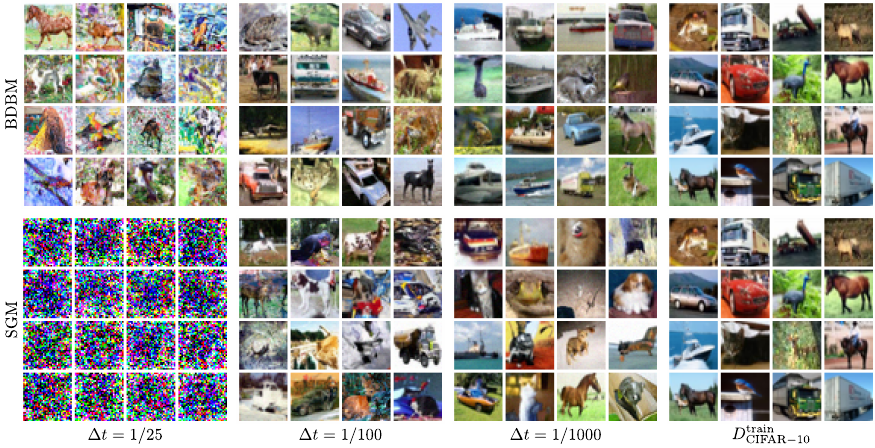}
\caption{Samples from the SGM and BDBM models with different discretization intervals, compared to samples from the CIFAR-10 training dataset.}\label{fig:generative_grid}
\end{figure}

\subsection{Dataset Transfer}\label{sec:application_matching}

In this section, we investigate an application where both the initial and the terminal distributions are complex, rendering SGM approaches inapplicable.
Following \citet{bortoli2021diffusion}, we explore the scenario where the initial distribution $Γ$ is represented by the MNIST dataset ($D_\mathrm{MNIST}$), whereas the terminal distribution $ϒ$ is derived from the first five lowercase and the first five uppercase characters, specifically, \texttt{a,…,e,A,…,E}, of the EMNIST dataset ($D_\mathrm{EMNIST}$).
Consequently, $d = 28 {×} 28$.
The reference process law $R$ is associated to $dX_t = σdW_t$ over the time interval $[0,1]$.
Our objective is to approximately solve $\DynSB$.

We evaluate the performance of the DIPF and IDBM procedures, alternating between solving the IDBM iterations in backward and forward time directions---a requirement for the DIPF procedure.
As in \citet{bortoli2021diffusion}, we utilize a lightweight configuration of the U-Net neural network architecture as proposed by \citet{dhariwal2021diffusion}.
We employ two neural network models $α(θ,x,t)$, $α(ϕ,x,t)$, which act as approximators to the drifts corresponding to the forward and backward iterations respectively.

Each of these models, $α(θ,x,t)$ and $α(ϕ,x,t)$, is associated with an independent instance of the Adam optimizer due to its adaptive nature, and a model copy ($α(\hat{θ},x,t)$ to $α(θ,x, t)$, $α(\hat{ϕ},x,t)$ to $α(ϕ,x,t)$).
The parameters of these copies, $\hat{θ}$ and $\hat{ϕ}$, are updated according to the Exponential Moving Averaging (EMA) scheme.
Models $α(\hat{θ},x,t)$ and $α(\hat{ϕ},x,t)$, whose parameters evolve more stably, are used to simulate the required SDE paths using the Euler scheme and a discretization interval of $Δt = 1/30$.
To increase efficiency, we implement caching of sampled paths.
For the DIPF algorithm, entire discretized paths are cached, while for the IDBM algorithm, only the initial and terminal values are cached.
Sampling from the reference diffusion bridge corresponding to $R$ at arbitrary time points can be performed quickly and exactly (\cref{sec:sde_class}).

We utilize the following training methods.
For the DIPF procedure, we rely on the drift matching estimator \cref{eq:drift}.
For the IDBM procedure, we employ the regression estimators \cref{eq:dbm_drift_inference}.
For the reference process $dX_t = σdW_t$, both the BDM \cref{eq:sde_dbm_simple} and BDBM \cref{eq:sde_bdbm_simple} have the same target $(X_1 - X_t)/(1 - t)$ entering the forward and backward losses in \cref{eq:dbm_drift_inference}.
Instead of using time-dependent regularizers as in \cref{sec:application_score_generative}, we limit the simulation of $t$ to $t ∼ 𝒰(0, 1 - Δt/2)$, which allows us to recover the Rectified Flow loss with $σ = 0$.

\begin{figure}
\centering
\includegraphics{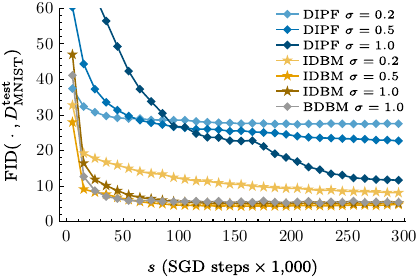}
\hfill
\includegraphics{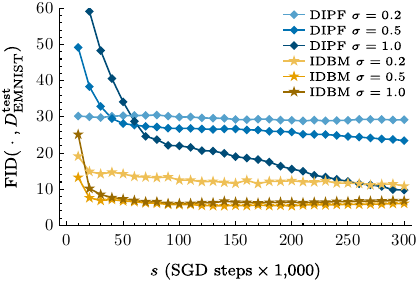}
\caption{IDBM and DIPF test FID values as function of SGD steps for (left): $D_\mathrm{EMNIST}^\mathrm{test}→D_\mathrm{MNIST}^\mathrm{test}$, i.e.\ backward time; (right): $D_\mathrm{MNIST}^\mathrm{test}→D_\mathrm{EMNIST}^\mathrm{test}$, i.e.\ forward time; linear-scale, truncated at $60$.}\label{fig:tansfer_bwd_fwd}
\end{figure}

Test FID values are calculated by initializing $X_0$ from $D_\mathrm{EMNIST}^\mathrm{test}$ ($D_\mathrm{MNIST}^\mathrm{test}$) and subsequently sampling the discretized path $X_t$ to obtain model samples that are approximately distributed as $D_\mathrm{MNIST}^\mathrm{test}$ ($D_\mathrm{EMNIST}^\mathrm{test}$).
The specific procedures corresponding to the IDBM and DIPF methods are as follows.
For the DIPF procedure, we remove the stochastic component from the terminal Euler discretization step as customary.
For the IDBM procedure, we employ the estimator $E_t ≔ 𝔼[X_1|X_t]$, which is directly obtained from $α(\hat{θ},x,t)$ and $α(\hat{ϕ},x,t)$.
Evaluating $E_{1 - Δt}$ is equivalent to removing the stochastic component from the terminal Euler discretization step, but we find that $E_{1 - 2 Δt}$ empirically performs better.

We consider three different levels of regularization $σ=1.0,0.5,0.2$, and compare the IDBM and DIPF procedures over $60$ iterations, with each iteration comprising $5,000$ SGD steps.
Additionally, we assess the simpler application of the BDBM procedure, which is trained for an equivalent amount of $300,000$ SGD steps.
Detailed settings of all hyperparameters defining this experiment, along with implementation details, are provided in the accompanying code implementation\footnote{\url{https://github.com/stepelu/idbm-pytorch}}.

\begin{table}
\centering
\begin{tabular}{@{}llllllll@{}}
\toprule
Procedure                             & \multicolumn{3}{c}{IDBM} & \multicolumn{3}{c}{DIPF} & \multicolumn{1}{c}{BDBM}                                 \\ \cmidrule(lr){2-4} \cmidrule(lr){5-7} \cmidrule(lr){8-8}
Regularization $σ$                    & $1.0$                    & $0.5$                    & $0.2$                    & $1.0$ & $0.5$ & $0.2$ & $1.0$ \\ \midrule
FID forward                           & 10.9                     & 6.1                      & 6.8                      & 29.2  & 23.4  & 9.8   &       \\
FID backward                          & 8.2                      & 4.9                      & 5.2                      & 27.6  & 22.8  & 11.7  & 5.3   \\
$𝕃_\mathrm{OC}(u,R^u,Σ_R)/d$ forward  & 4.0                      & 5.2                      & 12.5                     & 65.9  & 31.2  & 22.5  &       \\
$𝕃_\mathrm{OC}(u,R^u,Σ_R)/d$ backward & 3.9                      & 4.8                      & 12.5                     & 62.1  & 33.6  & 22.5  & 4.2   \\ \bottomrule
\end{tabular}
\caption{Test FID and $𝕃_\mathrm{OC}(u,R^u,Σ_R)$ values for fully trained IDBM, DIPF and BDBM procedures.}
\label{tab:tansfer_bwd_fwd}
\end{table}

The findings from this experiment are presented in \cref{tab:tansfer_bwd_fwd} and \cref{fig:tansfer_bwd_fwd}.
The IDBM procedure consistently displays superior convergence properties and lower test FID values compared to the DIPF procedure.
For both $σ=0.5$ and $σ=0.2$, the DIPF procedure fails to make significant progress, with model samples almost identical to the input samples.
Indeed, a FID value around $30$ corresponds to $\mathrm{FID}(D_\mathrm{MNIST}^\mathrm{test},D_\mathrm{EMNIST}^\mathrm{test})$.
Visualizations of trajectories of sampled paths $X_t$, and of estimators $E_t$ for the IDBM and BDBM approaches, are included in \cref{app:extra}.

The results from the IDBM procedure also illustrates a dependency on the regularization level $σ$, with performance deteriorating for the lowest regularization level $σ=0.2$.
We were unable to achieve satisfactory results for the case $σ=0$, which corresponds to the Rectified Flow method.
In this instance a valid transport is not achieved as well, with model samples strongly resembling the input samples.
We hypothesize that while the RF method offers the advantageous feature of progressively straightening the inferred paths, the ensuing inference problem may present considerable difficulty.
Indeed, the neural network model has to predict the terminal sample exactly at $t=0$ based on the initial sample, which commences to be the case at low levels of $σ$ (see $E_t$ in \cref{app:extra}).

The optimal SDE solving $\DynSB$ accomplishes a valid coupling while minimizing the drift norm functional $𝕃_\mathrm{OC}(u,R^u,Σ_R)$, as outlined in \cref{sec:reciprocal}.
Lower FID values serve as indicators of the accuracy of the inferred coupling.
As demonstrated by the results presented in \cref{tab:tansfer_bwd_fwd}, the IDBM procedure not only infers an accurate coupling but also effectively minimizes $𝕃_\mathrm{OC}(u,R^u,Σ_R)$.
Estimations of the test values for $𝕃_\mathrm{OC}(u,R^u,Σ_R)$ are computed through Monte Carlo sampling of the drift norm functional \cref{eq:sb_optimal_drift}.
In this process, the relevant neural network approximator is substituted for $μ(x,t)$.
Finally, the simpler BDBM procedure yields FID and $𝕃_\mathrm{OC}(u,R^u,Σ_R)$ values that are comparable to those produced by the more complex IDBM procedure.
This suggests that the BDBM crude approximation to $\DynSB$ might be acceptable within this context.

\section{Discussion}\label{sec:discussion}

This paper introduces a novel iterative algorithm, the IDBM procedure, aimed at solving the dynamic Schrödinger bridge problem $\DynSB$, and performs a preliminary theoretical analysis of its convergence properties.
The theoretical findings are complemented by various numerical simulations and analytical results, demonstrating the competitive performance of the IDBM procedure in comparison to the IPF procedure.

As in \citet{pavon2018datadriven,bortoli2021diffusion,vargas2021solving}, we assume that samples from the target initial and terminal distributions are either readily available, belonging to some dataset of interest, or can be generated without further approximations.
The IDBM procedure is particularly suitable in scenarios where the reference diffusion admits simple analytical transition densities.
These considerations suggest that the IDBM procedure is ideally suited for current generative applications.
Since the IDBM procedure produces a valid coupling at each iteration, we utilize the first IDBM iteration, i.e.\ the (backward) DBM transport, as an alternative to the approach of \citet{song2021scorebased}.
This alternative achieves accelerated training and superior sample quality at larger discretization intervals.
An additional advantage of this proposal is the simplicity of its implementation, which differs from the approach of \citet{song2021scorebased} only in the training loss definition.
The time-reversal sampling-based implementations of \citet{bortoli2021diffusion,vargas2021solving} suffer from the simulation-inference mismatch demonstrated in the present work, and might be difficult to scale to demanding generative applications due to the ensuing lower efficiency and higher computational cost.

The outcomes of our study give rise to additional avenues for future research.
From a theoretical perspective, our initial analysis falls short of establishing the convergence in KL divergence of the IDBM procedure's iterates to the solution to $\DynSB$.
The numerical and analytical results of \cref{sec:application_gaussian} support that, under suitable conditions, this stronger form of convergence is achieved.
Non-asymptotic results would be particularly valuable, as they could elucidate the conditions under which either the IPF or IDBM methods are more favorable.
Moreover, it would be desirable to introduce more easily verifiable conditions guaranteeing convergence of the IDBM procedure.
Finally, when $d > 1$, neither the IPF nor the IDBM procedures are robust to vanishing randomness in the reference diffusion $R$.
Consequently, the design of an iterative procedure solving $\DynSB$, while being robust to a vanishing regularization level, remains an open problem.

On the empirical side, computational constraints limited the scope of our simulation study comparing the DBM transport to the time-reversal approach of \citet{song2021scorebased}.
Specifically, no hyperparameter optimization was performed, neither for DBM-specific parameters, nor for the employed neural network architecture, while \cref{fig:generative_fid} suggests that the tested configuration might suffer from overfitting later in training.
This is conceivable, as the DBM transport approximates the solution to $\DynSB$, which is associated to an SDE that has a ``simpler'' optimal drift \citep{daipra1991stochastic}.
Since the exact solution to the DBM transport, and of the competing approaches, perfectly recovers the training data distribution, a less powerful architecture might be required for superior generalization properties.
Overfitting to the CIFAR-10 dataset is not a novel concern \citep{nichol2021improved}, and it is typically addressed through hyperparameter search.

At one extreme, when the reference process's dynamics imply a terminal value almost independent of the initial value, the DBM results in a generative model akin to that of \citet{song2021scorebased}, while rectifying the terminal distribution mismatch inherent in all time-reversal approaches \citep{bortoli2021diffusion}.
More generally, the DBM transport enlarges the space of feasible (exact) transports.
Consequently, it is plausible that improvements in current state-of-the-art solutions can be attained by empirically exploring this enlarged space, optimizing for a given metric of interest, such as FID.\@
In particular, high-resolution images pose challenges for time-reversal approaches and necessitate an increase in the reference process randomness, thus increasing the implicit integration time \citep{hoogeboom2023simple}.
The DBM transport, which precisely matches the initial and terminal marginal distributions, does not face this issue.
Lastly, computational constraints precluded an evaluation of the application of further iterations of the IDBM procedure in generative modeling applications, which would trade off increased training time with a more efficient generation process.

The approaches of \citet{bortoli2021diffusion,vargas2021solving} are directly applicable to a broad spectrum of reference SDEs.
Conversely, the IDBM procedure detailed in this work necessitates an analytically solvable reference SDE.\@
However, we remark that the IDBM methodology can potentially be extended to accommodate complex reference dynamics.
A promising approach involves harnessing the findings of \citet{debortoli2021simulating} to construct neural network approximators for $∇_x\log(r_{τ|t}(y|x))$ and $∇_y\log(r_{t|0}(y|x))$.
These approximators allow both the sampling of the reference diffusion bridge via a discretization of \cref{eq:bridge}, and the calculation of the targets in the estimators \cref{eq:dbm_drift_inference}.
An alternative strategy employs the Exact Algorithm (EA) proposed by \citet{beskos2005exact,beskos2008factorisation} to sample exactly from the reference diffusion bridge \cref{eq:bridge}.
The unavailability of $∇_x\log(r_{τ|t}(y|x))$ and $∇_y\log(r_{t|0}(y|x))$ can be circumvented by performing drift inference as outlined in \cref{sec:inference_simulation}, rather than relying on \cref{eq:dbm_drift_inference}.
We refer to the discussion following \cref{{eq:dbm_drift_inference}} for more details.
However, we must underscore that the application of the EA is restricted to reference SDEs with a ``reducible'' diffusion coefficient \citep{ait-sahalia2008closedform}, and to low-dimensional spaces due to scalability concerns \citep{peluchetti2012study}.
We refer to \citet{beskos2008factorisation} for the additional requirements on the drift coefficient of the reference SDE.\@
The empirical evaluation of these extensions is left to future research.

\newpage

\begin{appendices}
\crefalias{section}{appendix}

\section{Proofs}\label{app:proofs}

\subsection*{Proof of \cref{thm:mixture_matching}}

The following assumptions are sufficient for the application of \cref{thm:mixture_matching}.
We assume that \cref{eq:sde_mixture} and each member of \cref{eq:sde_family} admit a unique solution, with strictly positive marginal densities over $ℝ^d$.
We assume that each marginal density is the unique solution to the corresponding Fokker-Plank PDE, and that the exchanges of limits denoted with $(\star\star)$ hold.
When \cref{thm:mixture_matching} is applied to match $Π = CR_{•|0,τ}$ in \cref{sec:dbm_transports}, these conditions are satisfied if $R$ is given by \cref{eq:sde_class} and the marginal distributions of $C$ are given by mixtures of Gaussian distributions, which covers the setting of generative modelling applications.

\resmixturematching*

\begin{proof}
In this proof we make use of the following notation: for $f(x,t): ℝ^d × [0,τ] → ℝ$, ${(f(x,t))}_t ≔ \frac{d}{dt}f(x,t)$, for $f(x,t): ℝ^d × [0,τ] → ℝ^d$, ${(f(x,t))}_x ≔ \sum_{i=1}^{d}\frac{d}{dx_i}[f]_i(x,t)$, for $f(x,t): ℝ^d × [0,τ] → ℝ^d × ℝ^d$, ${(f(x,t))}_{xx} ≔ \sum_{i,j=1}^{d}\frac{d^2}{dx_i dx_j}[f]_{i, j}(x,t)$.
Let $Σ^λ(x,t) ≔ σ^λ(x,t)σ^λ(x,t)ᵀ$.
Then, for $0<t≤τ$,
\begin{align*}
 & {(π_t(x))}_t = {\left(∫_Λ p_t^λ(x) Ψ(dλ)\right)}_t                                                                                                                         \\
 & \quad = ∫_Λ{\left(p_t^λ(x)\right)}_t Ψ(dλ)\tag{$\star\star$}                                                                                                               \\
 & \quad = ∫_Λ{\left(\mu^λ(x,t)p_t^λ(x)\right)}_x + \frac{1}{2}{\left(Σ^λ(x,t)p_t^λ(x)\right)}_{xx} Ψ(dλ)                                                                     \\
 & \quad = ∫_Λ{\left(\frac{\mu^λ(x,t)p_t^λ(x)}{π_t(x)}π_t(x)\right)}_x + \frac{1}{2}{\left(\frac{Σ^λ(x,t)p_t^λ(x)}{π_t(x)}π_t(x)\right)}_{xx} Ψ(dλ)                           \\
 & \quad = {\left(∫_Λ\frac{\mu^λ(x,t)p_t^λ(x)}{π_t(x)}Ψ(dλ)π_t(x)\right)}_x + \frac{1}{2}{\left(∫_Λ\frac{Σ^λ(x,t)p_t^λ(x)}{π_t(x)}Ψ(dλ)π_t(x)\right)}_{xx}\tag{$\star\star$}.
\end{align*}
The lines denoted with $(\star\star)$ consist of exchange of limits, the third line results from the application of the Fokker-Plank PDEs for $\{P^λ\}_{λ∈Λ}$.
The Fokker-Plank PDE for \cref{eq:sde_mixture} is
\begin{equation*}
{(p_t(x))}_t = {\left(∫_Λ\frac{\mu^λ(x,t)p_t^λ(x)}{π_t(x)}Ψ(dλ)p_t(x)\right)}_x + \frac{1}{2}{\left(∫_Λ\frac{Σ^λ(x,t)p_t^λ(x)}{π_t(x)}Ψ(dλ)p_t(x)\right)}_{xx}.
\end{equation*}
As $P_0=Π_0$ and $p$ and $π$ satisfy the same Fokker-Plank PDE, it follows that $P$ and $Π$ share the same marginal distribution.
\end{proof}

\subsection*{Proof of \cref{thm:idbm_convergence}}

\residbmconvegence*

\begin{proof}
Let $P ∈ 𝒫_𝒞$ be the law of the diffusion solving
\begin{equation*}\begin{aligned}
 & dX_t = μ(X_t,t)dt + σ_R(X_t,t)dW_t,\quad t∈[0,τ], \\
 & X_0 ∼ P_0.
\end{aligned}\end{equation*}
Let $C ∈ 𝒫_2(Γ,ϒ)$.
Within this proof, we use the following abbreviations: $Π(C) ≔ Π(C,R_{•|0,τ})$, $M(C) ≔ M(Π(C,R_{•|0,τ}))$, $S^* ≔ S^*(Γ,ϒ,R,𝒫_𝒞)$.
Assume that $D_{KL}(Π(C) \TO P) < ∞$ and that the conditions required for the application of the Cameron-Martin-Girsanov theorem to obtain $dM(C)/dP$ are satisfied.
As $M(C) ≪ P$ as well, following \citet{csiszar1975divergence}, it holds that
\begin{equation*}\begin{aligned}
D_{KL}(Π(C) \TO P) - D_{KL}(Π(C) \TO M(C)) & = 𝔼_{Π(C)}\Big[\log\frac{dΠ(C)}{dP}\Big] - 𝔼_{Π(C)}\Big[\log\frac{dΠ(C)}{dM(C)}\Big] \\
                                           & = 𝔼_{Π(C)}\Big[\log\frac{dM(C)}{dP}\Big],
\end{aligned}\end{equation*}
and we want to show that
\begin{equation*}
𝔼_{Π(C)}\Big[\log\frac{dM(C)}{dP}\Big] = 𝔼_{M(C)}\Big[\log\frac{dM(C)}{dP}\Big].
\end{equation*}
As $Π_0(C) = M_0(C) = Γ$ it suffices to show that
\begin{equation*}
𝔼_{Π(C)}\Big[\log\frac{dM_{•|0}(C)}{dP_{•|0}}\Big] = 𝔼_{M(C)}\Big[\log\frac{dM_{•|0}(C)}{dP_{•|0}}\Big],
\end{equation*}
where
\begin{equation*}
\log\frac{dM_{•|0}(C)}{dP_{•|0}}(x) = ∫_0^τ[(μ_M - μ)ᵀΣ_R^{-1}](x_t,t)dx_t - \frac{1}{2} ∫_0^τ [(μ_M - μ)ᵀΣ_R^{-1}(μ_M + μ)](x_t,t) dt
\end{equation*}
is given by \cref{eq:girsanov}.
As $Π(δ_{x_0}{⊗}ϒ,R_{•|0,τ}) = Π_{•|0}(C,R_{•|0,τ})$, from \cref{eq:dbm_fwd_x0},
\begin{equation*}\begin{aligned}
 & 𝔼_{Π(C)}\Big[∫_0^τ[(μ_M - μ)ᵀΣ_R^{-1}](X_t,t)dX_t\Big]                                                                                       \\
 & \quad = 𝔼_{Π(C)}\Big[∫_0^τ[(μ_M - μ)ᵀΣ_R^{-1}](X_t,t)\big(μ_R(X_t,t) + Σ_R(X_t,t) 𝔼_{Π(C)}[∇_{X_t}\log r_{τ|t}(X_τ|X_t)|X_t,X_0]\big)dt\Big] \\
 & \quad = 𝔼_{Π(C)}\Big[∫_0^τ[(μ_M - μ)ᵀΣ_R^{-1}](X_t,t)\big(μ_R(X_t,t) + Σ_R(X_t,t) 𝔼_{Π(C)}[∇_{X_t}\log r_{τ|t}(X_τ|X_t)|X_t]\big)dt\Big],
\end{aligned}\end{equation*}
by the tower property of conditional expectations.
On the other hand, from \cref{eq:dbm_fwd},
\begin{equation*}\begin{aligned}
 & 𝔼_{M(C)}\Big[∫_0^τ[(μ_M - μ)ᵀΣ_R^{-1}](X_t,t)dX_t\Big]                                                                                    \\
 & \quad = 𝔼_{M(C)}\Big[∫_0^τ[(μ_M - μ)ᵀΣ_R^{-1}](X_t,t)\big(μ_R(X_t,t) + Σ_R(X_t,t) 𝔼_{Π(C)}[∇_{X_t}\log r_{τ|t}(X_τ|X_t)|X_t]\big)dt\Big]  \\
 & \quad = 𝔼_{Π(C)}\Big[∫_0^τ[(μ_M - μ)ᵀΣ_R^{-1}](X_t,t)\big(μ_R(X_t,t) + Σ_R(X_t,t) 𝔼_{Π(C)}[∇_{X_t}\log r_{τ|t}(X_τ|X_t)|X_t]\big)dt\Big],
\end{aligned}\end{equation*}
as $M_t(C) = Π_t(C)$ for all $t ∈ [0,τ]$.
In the same way, equality in expectations of $- \frac{1}{2} ∫_0^τ [(μ_M - μ)ᵀΣ_R^{-1}(μ_M + μ)](X_t,t) dt$ is established.
We thus obtain a version of the Pythagorean law for (reverse) KL-projections \citep{csiszar1975divergence,nielsen2018what},
\begin{equation*}
D_{KL}(Π(C) \TO P) = D_{KL}(Π(C) \TO M(C)) + D_{KL}(M(C) \TO P).
\end{equation*}
See also \citet{liu2023learning} for another derivation of this result.
When $P = S^*$ ($S^*$ is law of the diffusion solving \cref{eq:schrodinger_sde}), under the same assumptions,
\begin{equation}\label{eq:pi_to_m_ineq}
D_{KL}(Π(C) \TO S^*) ≥ D_{KL}(M(C) \TO S^*),
\end{equation}
with equality if and only if $Π(C) = M(C)$.

If $Π(C) = S^*$, then $M(C) = S^*$ as $Π(C) = CR_{•|0,τ}$ is equal in law to the diffusion process solving \cref{eq:schrodinger_sde} and \cref{thm:mixture_matching} applied to a (single) diffusion results in that same diffusion.
On the other hand, if $Π(C) ≠ S^*$, then $Π(C)$ is not a diffusion (it is an Ito process in the sense of \citet{oksendal2003stochastic}), see \cref{sec:reciprocal}, and $M(C) ≠ Π(C)$ \citep[Theorem 8.4.3]{oksendal2003stochastic}.
That is $Π(C) = M(C)$ if and only if $Π(C) = M(C) = S^*$.

For $i ≥ 0$,
\begin{equation*}\begin{aligned}
D_{KL}(Π^{(i)} \TO S^*) & = D_{KL}(Π^{(i)}_{0,τ} \TO S^*_{0,τ}) + 𝔼_{Π^{(i)}_{0,τ}}[D_{KL}(Π^{(i)}_{•|0,τ} \TO S^*_{•|0,τ})] \\
                        & = D_{KL}(Π^{(i)}_{0,τ} \TO S^*_{0,τ}),
\end{aligned}\end{equation*}
as $Π^{(i)}_{•|0,τ} = S^*_{•|0,τ} = R_{•|0,τ}$, and
\begin{equation*}\begin{aligned}
D_{KL}(M^{(i)} \TO S^*) & = D_{KL}(M^{(i)}_{0,τ} \TO S^*_{0,τ}) + 𝔼_{M^{(i)}_{0,τ}}[D_{KL}(M^{(i)}_{•|0,τ} \TO S^*_{•|0,τ})] \\
                        & ≥ D_{KL}(Π^{(i+1)}_{0,τ} \TO S^*_{0,τ}),
\end{aligned}\end{equation*}
as $M^{(i)}_{0,τ} = Π^{(i+1)}_{0,τ}$.
Thus,
\begin{equation}\label{eq:m_to_pi_ineq}
D_{KL}(M^{(i)} \TO S^*) ≥ D_{KL}(Π^{(i+1)} \TO S^*).
\end{equation}

It is assumed that $D_{KL}(Π^{(0)} \TO S^*) = D_{KL}(C^{(0)} \TO S^*_{0,τ}) < ∞$, and that the Cameron-Martin-Girsanov theorem applies to each $M^{(i)}$ yielding $dM^{(i)}/dS^*$.
Therefore, we can iteratively apply \cref{eq:pi_to_m_ineq,eq:m_to_pi_ineq} obtaining $D_{KL}(Π^{(i)} \TO S^*) ≥ D_{KL}(M^{(i)} \TO S^*) ≥ D_{KL}(Π^{(i+1)} \TO S^*)$ for $i ≥ 0$.
Being non-increasing and bounded below, the sequence $D_{KL}(Π^{(i)} \TO S^*), D_{KL}(M^{(i)} \TO S^*), D_{KL}(Π^{(i+1)} \TO S^*)$ converges, and $\lim_{i → ∞} (D_{KL}(Π^{(i)} \TO S^*) - D_{KL}(M^{(i)} \TO S^*)) = \lim_{i → ∞} D_{KL}(Π^{(i)} \TO M^{{i}}) = 0$.

The sequences $\{Π^{(i)}\}_{i ≥ 0}$, $\{M^{(i)}\}_{i ≥ 0}$ are tight.
We consider $\{Π^{(i)}\}_{i ≥ 0}$, the case of $\{M^{(i)}\}_{i ≥ 0}$ being identical.
By the conditional Jensen inequality, for any measurable set $K$,
\begin{equation*}\begin{aligned}
D_{KL}(Π^{(i)} \TO S^*) & = 𝔼_{Π^{(i)}}\Big[-\log\frac{dS^*}{dΠ^{(i)}}\mathrel{\Big|}K^c\Big]Π^{(i)}[K^c] + 𝔼_{Π^{(i)}}\Big[-\log\frac{dS^*}{dΠ^{(i)}}\mathrel{\Big|}K\Big]Π^{(i)}[K] \\
                        & ≥ -\log(S^*[K^c]/Π^{(i)}[K^c])Π^{(i)}[K^c] + -\log(S^*[K]/Π^{(i)}[K])Π^{(i)}[K].
\end{aligned}\end{equation*}
For $ε > 0$, choose $K$ compact such $S^*[K^c] < ε$, $S^*[K] ≥ 1-ε$ by the tightness of $S^*$ (as it is defined on a Polish space, see \citet{leonard2014properties}).
Assume that $Π^{(i)}$ is not tight.
That is, there is $γ > 0$ such that for each compact $K$ there is $i ≥ 0$ with $Π^{(i)}[K^c] ≥ γ, Π^{(i)}[K] < 1-γ$.
Hence, for each $ε > 0$, there is $i ≥ 0$ such that
\begin{equation*}
-\log(S^*[K^c]/Π^{(i)}[K^c])Π^{(i)}[K^c] ≥ -\log(ε/γ)γ,
\end{equation*}
which can be made arbitrarily large by a suitable small $ε > 0$.
On the other hand, $-\log(S^*[K]/x)x$ is bounded below for $x∈[0,1]$, so $-\log(S^*[K]/Π^{(i)}[K])Π^{(i)}[K]$ is bounded below.
But we know that for $i$ large enough $D_{KL}(Π^{(i)} \TO S^*)$ is bounded above, hence $Π^{(i)}$ must be tight.

Therefore, $\{Π^{(i)}\}_{i ≥ 0}$ and $\{M^{(i)}\}_{i ≥ 0}$ are relatively compact.
Each subsequence of $\{Π^{(i)}\}_{i ≥ 0}$ has a further subsequence $\{Π^{(l)}\}_{l ≥ 1}$ which converges in law: $Π^{(l)} \overset{ℒ}{⟶} Π^{(∞)}$ as $l→∞$ for some $Π^{(∞)} ∈ 𝒫_𝒞(Γ,ϒ)$.
Each subsequence of $\{M^{(i)}\}_{i ≥ 0}$ has a further subsequence $\{M^{(l)}\}_{l ≥ 1}$ which converges in law: $M^{(l)} \overset{ℒ}{⟶} M^{(∞)}$ for some $M^{(∞)} ∈ 𝒫_𝒞(Γ,ϒ)$.
By the lower semi-continuity of the KL divergence with respect to the topology of weak convergence \citep[Theorem 19]{vanerven2014renyi}, $0 = \liminf_{l→∞} D_{KL}(Π^{(l)} \TO M^{{i}}) ≥ D_{KL}(Π^{(∞)} \TO M^{{∞}})$ and therefore $Π^{(∞)} = M^{(∞)}$.
Denote this common PM with $S^{(∞)}$.

Due to the convergence in law of $Π^{(l)}$, and the form of its disintegration, $S^{(∞)}$ is of the form $S^{(∞)}_{0,τ}R_{•|0,τ}$ for some $S^{(∞)}_{0,τ} ∈ 𝒫_2(Γ,ϒ)$.
If $S^{(∞)}$ is also diffusion, it follows that $S^{(∞)}_{0,τ} = S^{(*)}_{0,τ}$ (\cref{sec:reciprocal}) and thus $S^{(∞)} = S^*$.
It remains to show that $S^{(∞)}$ is a diffusion when $σ_R(x,t) = I$.
Again by the lower semi-continuity of the KL divergence, we have that $D_{KL}(S^{(∞)} \TO S^*) < ∞$ and thus $S^{(∞)} ≪ S^*$.
But $S^* ≪ R$ and $R ≪ R^{\circ}$ by assumption where $R^{\circ}$ is the $d$-dimensional Brownian measure on $[0,τ]$.
Then, by \citet[Theorem 7.11 and following Note 1]{liptser1977statistics}, $S^{(∞)}$ is a diffusion.

As convergence in law to $S^*$ has been established for arbitrary convergent subsequences, $\{Π^{(i)}\} \overset{ℒ}{⟶} S^{(∞)}$ and $\{M^{(i)}\} \overset{ℒ}{⟶} S^{(∞)}$ by \citet[Theorem 2.6]{billingsley1999convergence}.
\end{proof}

\section{Additional SDE Class Formulae}\label{app:sde_class_extra}

The linearity of \cref{eq:sde_class_x_score_s,eq:sde_class_x_score_t} stands behind the form of the conditional expectations entering the drift coefficients that follow.
The generative time reversal process \cref{eq:sde_time_reversal} is given by
\begin{equation*}\begin{aligned}
 & d\RX_t = \Big[αβ_𝔯\RX_t + β_𝔯 \frac{𝔼_R[\RX_τ|\RX_t]a(0,b_𝔯) - \RX_t}{v(0,b_𝔯)}\Big]dt + \sqrt{β_𝔯}Γ^{1/2}dW_t,\quad t ∈ [0,τ], \\
 & \RX_0 ∼ R_τ ≈ ϒ,
\end{aligned}\end{equation*}
the DBM transport \cref{eq:dbm_fwd} is given by
\begin{equation*}\begin{aligned}
 & dX_t = \Big[-αβ_t X_t + β_t \frac{𝔼_{Π}[X_τ|X_t]a(b_t,b_τ) - X_t a(b_t,b_τ)^2}{v(b_t,b_τ)}\Big]dt + \sqrt{β_t}Γ^{1/2}dW_t,\quad t ∈ [0,τ], \\
 & X_0 ∼ Γ,
\end{aligned}\end{equation*}
while the BDBM transport \cref{eq:dbm_bwd} is given by
\begin{equation*}\begin{aligned}
 & d\RX_t = \Big[αβ_t \RX_t + β_𝔯 \frac{𝔼_{Π}[\RX_τ|\RX_t]a(0,b_𝔯) - \RX_t}{v(0,b_𝔯)} \Big]dt + \sqrt{β_𝔯}Γ^{1/2}dW_t,\quad t ∈ [0,τ], \\
 & \RX_0 ∼ ϒ.
\end{aligned}\end{equation*}
For $α=0$, the DBM and BDBM transports are symmetric in the following sense: the BDM transport based on $C ∈ 𝒫_2(Γ,ϒ)$ and $β_t$ is equivalent in law to the BDBM transport based on $\RC ∈ 𝒫_2(ϒ,Γ)$ and $\overline{β}_t ≔ β_𝔯$.

\section{Generative Modeling Code}\label{app:score_generative}

\begin{listing}
\begin{minted}[linenos,numbersep=5pt,fontsize=\scriptsize,mathescape,frame=lines,framesep=2mm,highlightlines={15-16},highlightcolor=easy_red_bg]{python}
# requires:
#     b_t(t): $b_t(t)$
#     a_s_t(s, t): $a(s, t)$ from (29)
#     v_s_t(s, t): $v(s, t)$ from (29)
# inputs:
#     x_0: [B,C,H,W]
#     model: ([B,C,H,W], [B]) -> [B,C,H,W]
# outputs:
#     losses: [B]
def sgm_loss(x_0, model, T=1.0):
    t = torch.rand(x_0.shape[0], device=x_0.device) * T                              # [B]
    scaled_t = b_t(t)                                                                # [B]
    z = torch.randn_like(x_0)                                                        # [B,C,H,W]
    a_0_t, v_0_t = a_s_t(0.0, scaled_t), v_s_t(0.0, scaled_t)                        # [B]
    s_0_t = torch.sqrt(v_0_t)                                                        # [B]
    x_t = a_0_t[:, None, None, None] * x_0 + s_0_t[:, None, None, None] * z          # [B,C,H,W]
    score_t = (x_0 * a_0_t[:, None, None, None] - x_t) / v_0_t[:, None, None, None]  # [B,C,H,W]
    losses = (model(x_t, scaled_t) - score_t)**2 * v_0_t[:, None, None, None]        # [B,C,H,W]
    losses = torch.sum(losses.view(losses.shape[0], -1), dim=1)                      # [B]
    return losses
\end{minted}
\caption{SGM loss sampling; \texttt{x\_0} is a batch of \texttt{B} images, each of which has \texttt{C} channels, height \texttt{H} and width \texttt{W}.}\label{code:sgm}
\end{listing}

\begin{listing}
\begin{minted}[linenos,numbersep=5pt,fontsize=\scriptsize,mathescape,frame=lines,framesep=2mm,highlightlines={18-26},highlightcolor=easy_blue_bg]{python}
# requires:
#     b_t(t): $b_t(t)$
#     a_s_t(s, t): $a(s, t)$ from (29)
#     v_s_t(s, t): $v(s, t)$ from (29)
#     al_s_t_u(s, t, u): $\hat{a}(s,t,u)$ from (32)
#     ar_s_t_u(s, t, u): $\check{a}(s,t,u)$ from (32)
#     vb_s_t_u(s, t, u): $\tilde{v}(s,t,u)$ from (32)
# inputs:
#     x_0: [B,C,H,W]
#     model: ([B,C,H,W], [B]) -> [B,C,H,W]
# outputs:
#     losses: [B]
def bdbm_loss(x_0, model, T=1.0, sigma_T=1.0):
    t = torch.rand(x_0.shape[0], device=x_0.device) * T                              # [B]
    scaled_t, scaled_T = b_t(t), b_t(T)                                              # [B]
    z = torch.randn_like(x_0)                                                        # [B,C,H,W]
    a_0_t, v_0_t = a_s_t(0.0, scaled_t), v_s_t(0.0, scaled_t)                        # [B]
    x_T = torch.randn_like(x_0) * sigma_T                                            # [B,C,H,W]
    al_0_t_T,  ar_0_t_T, = al_s_t_u(0.0, t, scaled_T), ar_s_t_u(0.0, t, scaled_T)    # [B]
    vb_0_t_T = vb_s_t_u(0.0, t, scaled_T)                                            # [B]
    sb_0_t_T = torch.sqrt(vb_0_t_T)                                                  # [B]
    x_t = (
        + al_0_t_T[:, None, None, None] * x_0
        + ar_0_t_T[:, None, None, None] * x_T
        + sb_0_t_T[:, None, None, None] * z
    )                                                                                # [B,C,H,W]
    score_t = (x_0 * a_0_t[:, None, None, None] - x_t) / v_0_t[:, None, None, None]  # [B,C,H,W]
    losses = (model(x_t, scaled_t) - score_t)**2 * v_0_t[:, None, None, None]        # [B,C,H,W]
    losses = torch.sum(losses.view(losses.shape[0], -1), dim=1)                      # [B]
    return losses
\end{minted}
\caption{BDBM loss sampling; \texttt{x\_0} is a batch of \texttt{B} images, each of which has \texttt{C} channels, height \texttt{H} and width \texttt{W}.}\label{code:bdbm}
\end{listing}

\section{Additional Experimental Results}\label{app:extra}

\begin{figure}[H]
\centering
\includegraphics[width=0.75\linewidth]{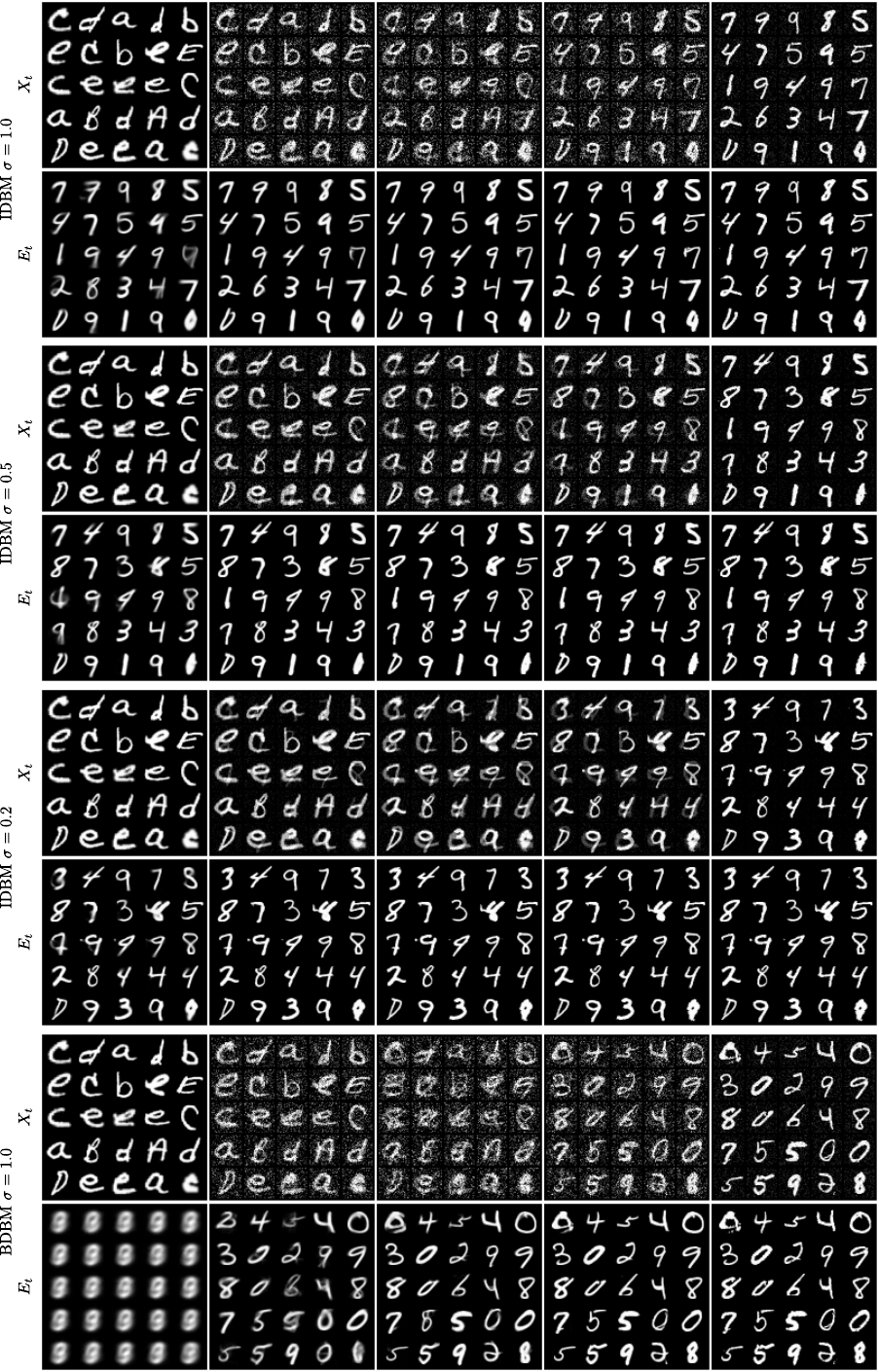}
\caption{Random paths $X_t$ and corresponding terminal value estimators $E_t$ for fully trained BDBM and IDBM procedures corresponding to different regularization levels; backward time direction; $t$ uniformly spaced on $[0,1]$ for $X_t$, $[0, 1 - Δt]$  for $E_t$; $X_0 ∼ D_\mathrm{EMNIST}^\mathrm{test}$ is fixed.}\label{fig:idbm_grid_bwd}
\end{figure}

\begin{figure}[H]
\centering
\includegraphics[width=0.75\linewidth]{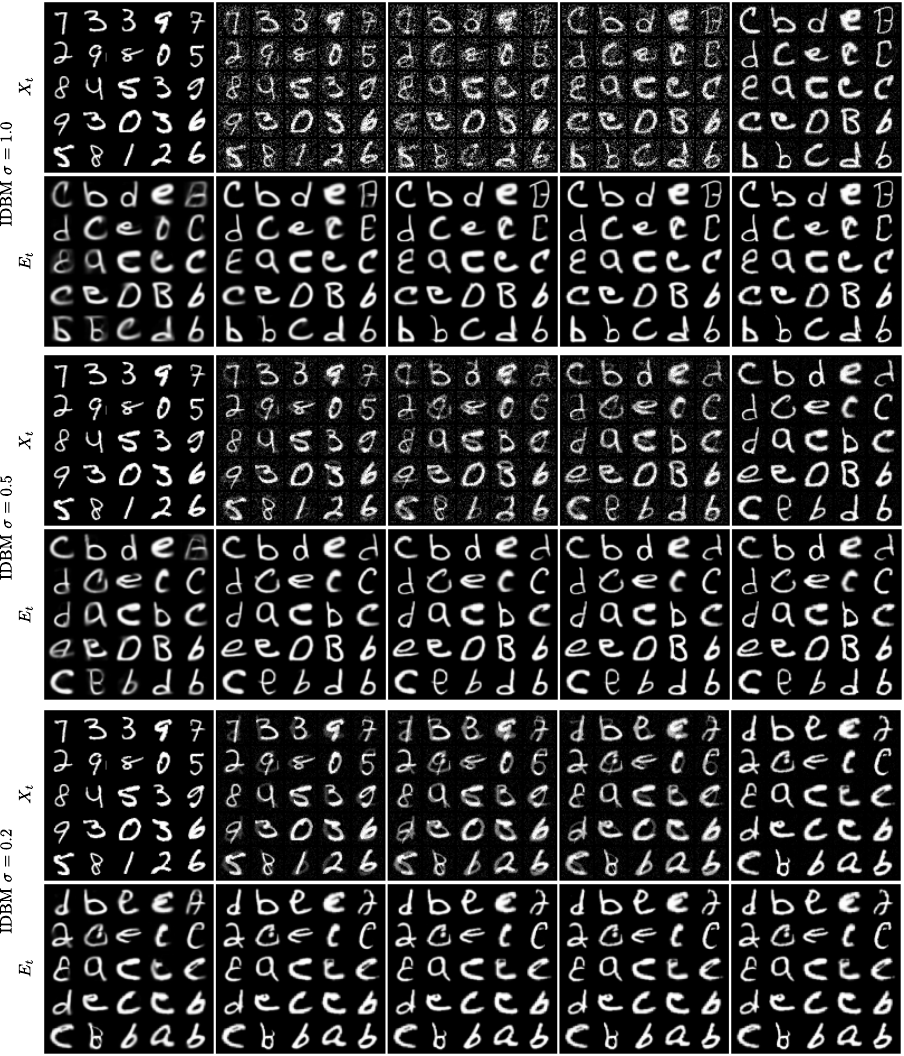}
\caption{Random paths $X_t$ and corresponding terminal value estimators $E_t$ for fully trained IDBM procedures corresponding to different regularization levels; forward time direction; $t$ uniformly spaced on $[0,1]$ for $X_t$, $[0, 1 - Δt]$ for $E_t$; $X_0 ∼ D_\mathrm{MNIST}^\mathrm{test}$ is fixed.}\label{fig:idbm_grid_fwd}
\end{figure}

\begin{figure}[H]
\centering
\includegraphics[width=0.75\linewidth]{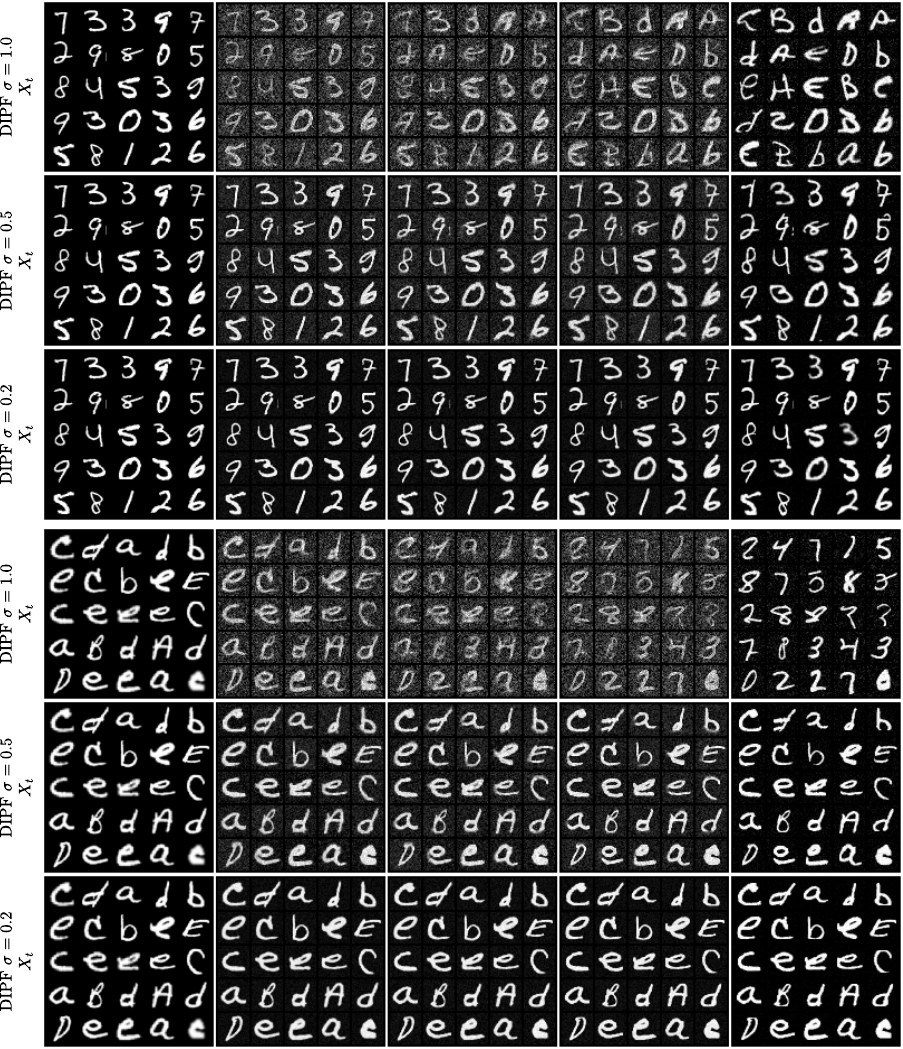}
\caption{Random paths $X_t$ for fully trained DIPF procedures corresponding to different regularization levels; forward (first three rows) and backward (last three rows) time directions; $t$ uniformly spaced on $[0,1]$; $X_0 ∼ D_\mathrm{MNIST}^\mathrm{test}$ (forward time) and $X_0 ∼ D_\mathrm{EMNIST}^\mathrm{test}$ (backward time) are fixed.}\label{fig:dipf_grid}
\end{figure}

\section{Notation}\label{app:notation}

\begin{table}[H]
\centering
\resizebox{0.95\linewidth}{!}{
\begin{tabular}{ll}
\toprule
Notation                         & Description                                                                        \\
\midrule
$[0,τ]$                          & Time interval                                                                      \\
$d$                              & Dimension of state space                                                           \\
$P$, $Q$, $…$                    & Probability measure (PM)                                                           \\
$𝒫_𝒞$                            & PMs over $𝒞([0,τ],ℝ^d)$                                                            \\
$𝒫_n$                            & PMs over $ℝ^{d × n}$                                                               \\
$𝒫_𝒞(Γ,ϒ)$                       & PMs with initial-terminal distributions $Γ,ϒ$; $𝒫_𝒞(Γ,ϒ) ⊆ 𝒫_𝒞$                    \\
$𝒫_2(Γ,ϒ)$                       & PMs with marginal distributions $Γ,ϒ$; $𝒫_2(Γ,ϒ) ⊆ 𝒫_2$                            \\
$𝒰(0,τ)$                         & Uniform distribution on $(0,τ)$                                                    \\
$𝒩_d(μ,Σ)$                       & $d$-variate normal distribution with mean $μ$ and covariance $Σ$                   \\
$dP/dQ$                          & Density of PM $P$ with respect to PM $Q$                                           \\
$p$                              & Lebesgue density of PM $P$                                                         \\
$P_{t_1,…,t_n}$                  & Marginalization of $P ∈ 𝒫_𝒞$ at $t_1,…,t_n$                                        \\
$P_{•|t_1,…,t_n}$                & Conditioning of $P ∈ 𝒫_𝒞$ given values at $t_1,…,t_n$                              \\
$P=P_{t_1,…,t_n}P_{•|t_1,…,t_n}$ & Marginal-conditional decomposition of $P ∈ 𝒫_𝒞$                                    \\
$Q=ΨP_{•|t_1,…,t_n}$             & Mixing of $P ∈ 𝒫_𝒞$ via $Ψ ∈ 𝒫_n$ at times $t_1,…,t_n$; $Q ∈ 𝒫_𝒞$                  \\
$D_{KL}(S \TO R)$                & Kullback-Leibler divergence from PM $S$ to PM $R$                                  \\
$I$                              & Identity matrix                                                                    \\
$Aᵀ$                             & Transposition of a square matrix $A$                                               \\
$‖V‖$                            & Euclidean norm of a vector $V$                                                     \\
$\RX$                            & Time reversal of stochastic process $X$                                            \\
$\RP$                            & Time reversal of PM $P$; $P ∈ 𝒫_𝒞$                                                 \\
$𝔯$                              & Reverse timescale; $𝔯 ≔ τ - t$                                                     \\
$W$                              & $d$-dimensional standard Brownian motion                                           \\
\midrule
$Γ$, $ϒ$                         & Target PMs; $Γ,ϒ ∈ 𝒫_1$                                                            \\
$R$                              & Reference PM;\@ $R ∈ 𝒫_𝒞$                                                          \\
$μ_R(x,t)$, $σ_R(x,t)$           & Reference SDE drift and diffusion coefficients                                     \\
$Σ_R(x,t)$                       & Reference SDE covariance coefficient; $Σ_R ≔ σ_Rσ_Rᵀ$                              \\
$S^*(Γ,ϒ,R,𝒫_𝒞)$                 & Solution $\DynSB$ (dynamic)                                                        \\
$S^*(Γ,ϒ,R_{0,τ},𝒫_2)$           & Solution to $\StaSB$ (static) (corresponding to $\DynSB$)                          \\
$S^*(Γ,□,Q,𝒫)$                   & Solution to forward $\DynHSB$ ($𝒫=𝒫_𝒞$), $\StaHSB$ ($𝒫=𝒫_2$); $Q∈𝒫$                \\
$S^*(□,ϒ,Q,𝒫)$                   & Solution to backward $\DynHSB$ ($𝒫=𝒫_𝒞$), $\StaHSB$ ($𝒫=𝒫_2$); $Q∈𝒫$               \\
$F^{(i)}$                        & $i$-th IPF iteration; $F^{(i)} ∈ 𝒫_𝒞 ∪ 𝒫_2$, $i ≥ 0$                               \\
$μ_{F^{(i)}}$                    & Drift corresponding to $F^{(i)} ∈ 𝒫_C$                                             \\
$Π(C,R_{•|0,τ})$                 & Diffusion mixture; $C ∈ 𝒫_2(Γ,ϒ)$, $Π ∈ 𝒫_𝒞(Γ,ϒ)$                                  \\
$M(Π(C,R_{•|0,τ}))$              & Forward DBM based on $C$; $C ∈ 𝒫_2(Γ,ϒ)$, $M ∈ 𝒫_𝒞(Γ,ϒ)$                           \\
$\RM(Π(C,R_{•|0,τ}))$            & Backward DBM based on $C$; $C ∈ 𝒫_2(Γ,ϒ)$, $\RM ∈ 𝒫_𝒞(ϒ,Γ)$                        \\
$M^{(i)}$                        & $i$-th (forward) IDBM iteration; $M^{(i)} ∈ 𝒫_𝒞$, $i ≥ 0$                          \\
$C^{(i)}$                        & $i$-th IDBM coupling; $C^{(i)} ∈ 𝒫_2(Γ,ϒ)$, $C^{(i)} ≔ M^{(i)}_{0,τ}$              \\
$Π^{(i)}$                        & $i$-th IDBM diffusion mixture; $Π^{(i)} ∈ 𝒫_𝒞$, $Π^{(i)} ≔ Π(C^{(i-1)},R_{•|0,τ})$ \\
$μ_{M^{(i)}}$                    & Drift corresponding to $M^{(i)} ∈ 𝒫_𝒞$                                             \\
\bottomrule
\end{tabular}
}
\caption{Main notation summary.}\label{tab:symbols}
\end{table}

\end{appendices}

\clearpage
\bibliography{references}

\end{document}